\documentclass[accepted]{uai2025} 
                        
\usepackage[american]{babel}

\usepackage{natbib} 
\renewcommand{\bibsection}{\subsubsection*{References}}
\usepackage{mathtools} 
\usepackage{booktabs} 
\usepackage{tikz} 
\usepackage{graphicx}

\usepackage{amsthm}
\usepackage{tablefootnote}
\usepackage{wrapfig}
\usepackage{subcaption}
\usepackage{lipsum}
\usepackage{enumitem}
\usepackage[ruled,linesnumbered]{algorithm2e}
\usepackage{multicol}
\usepackage{multirow}
\usepackage{array}

\usetikzlibrary{arrows.meta}
\usetikzlibrary{decorations.pathreplacing}
\usetikzlibrary{positioning, fit}
\tikzset{%
  do path picture/.style={%
    path picture={%
      \pgfpointdiff{\pgfpointanchor{path picture bounding box}{south west}}%
        {\pgfpointanchor{path picture bounding box}{north east}}%
      \pgfgetlastxy\x\y%
      \tikzset{x=\x/2,y=\y/2}%
      #1
    }
  },
  sin wave/.style={do path picture={    
    \draw [line cap=round] (-3/4,0)
      sin (-3/8,1/2) cos (0,0) sin (3/8,-1/2) cos (3/4,0);
  }},
  cross/.style={do path picture={    
    \draw [line cap=round] (-1,-1) -- (1,1) (-1,1) -- (1,-1);
  }},
  plus/.style={do path picture={    
    \draw [line cap=round] (-3/4,0) -- (3/4,0) (0,-3/4) -- (0,3/4);
  }}
}

\usepackage[capitalise]{cleveref}
\crefname{equation}{Eq.}{Eqs.}
\crefname{figure}{Fig.}{Figs.}
\crefname{table}{Tab.}{Tabs.}
\crefname{section}{Sec.}{Secs.}
\crefname{appendix}{App.}{Apps.}
\crefname{algorithm}{Alg.}{Algs.}
\crefname{lemma}{Lem.}{Lems.}
\crefname{theorem}{Thm.}{Thms}
\crefname{definition}{Def.}{Defs}
\crefname{proposition}{Prop.}{Props.}
\crefname{corollary}{Cor.}{Cors.}
\crefname{remark}{Rmk.}{Rmks}

\newtheorem{theorem}{Theorem}[section]

\newtheorem{proposition}[theorem]{Proposition}

\usepackage{amsmath,amsfonts,bm,amssymb}

\newcommand{\unif}{\mathrm{Unif}}

\def\eqref#1{(\ref{#1})}

\newcommand{\dif}{{\mathrm{d}}}

\def\ve{{\bm{e}}}

\def\vw{{\bm{w}}}
\def\vx{{\bm{x}}}
\def\vy{{\bm{y}}}
\def\vz{{\bm{z}}}

\def\valpha{{\bm{\alpha}}}
\def\vbeta{{\bm{\beta}}}

\def\vpi{{\bm{\pi}}}

\DeclareMathAlphabet{\mathsfit}{\encodingdefault}{\sfdefault}{m}{sl}
\SetMathAlphabet{\mathsfit}{bold}{\encodingdefault}{\sfdefault}{bx}{n}

\def\gD{{\mathcal{D}}}

\def\gG{{\mathcal{G}}}

\def\gP{{\mathcal{P}}}

\newcommand{\etens}[1]{\mathsfit{#1}}

\def\ety{{\etens{y}}}

\newcommand{\E}{\mathbb{E}}
\renewcommand{\P}{\mathbb{P}}
\newcommand{\Ls}{\mathcal{L}}
\newcommand{\R}{\mathbb{R}}

\newcommand{\softmax}{\mathrm{softmax}}

\newcommand{\KL}{D_{\mathrm{KL}}}
\newcommand{\var}{\mathrm{var}}

\DeclareMathOperator*{\argmax}{arg\,max}
\DeclareMathOperator*{\argmin}{arg\,min}

\DeclareMathOperator{\diag}{diag}

\newcommand{\hyphen}{{\text{-}}}

\title{COS-DPO: Conditioned One-Shot Multi-Objective Fine-Tuning Framework}

\author[1]{\href{mailto:<yinuoren@stanford.edu>?Subject=Your UAI 2025 paper}{Yinuo~Ren}{}}
\author[2]{Tesi~Xiao}
\author[2]{Michael~Shavlovsky}
\author[3,1]{Lexing~Ying}
\author[2]{Holakou~Rahmanian}
\affil[1]{%
    Institute for Computational and Mathematical Engineering (ICME), Stanford University
}
\affil[2]{%
    Amazon
}
\affil[3]{%
    Department of Mathematics, Stanford University
}

\begin{document}
\maketitle

\begin{abstract}
    In LLM alignment and many other ML applications, one often faces the \emph{Multi-Objective Fine-Tuning} (MOFT) problem, \emph{i.e.}, fine-tuning an existing model with datasets labeled w.r.t. different objectives simultaneously.
    To address the challenge, we propose a \emph{Conditioned One-Shot} fine-tuning framework (COS-DPO) that extends the Direct Preference Optimization technique, originally developed for efficient LLM alignment with preference data, to accommodate the MOFT settings. 
    By direct conditioning on the weight across auxiliary objectives, our Weight-COS-DPO method enjoys an efficient one-shot training process for profiling the Pareto front and is capable of achieving comprehensive trade-off solutions even in the post-training stage.
    Based on our theoretical findings on the linear transformation properties of the loss function, we further propose the Temperature-COS-DPO method that augments the temperature parameter to the model input, enhancing the flexibility of post-training control over the trade-offs between the main and auxiliary objectives.
    We demonstrate the effectiveness and efficiency of the COS-DPO framework through its applications to various tasks, including the Learning-to-Rank (LTR) and LLM alignment tasks, highlighting its viability for large-scale ML deployments. 
\end{abstract}

\section{Introduction}

\emph{Direct Preference Optimization (DPO)}~\citep{rafailov2024direct} has been introduced as a memory- and computation-efficient alternative to the traditional \emph{Reinforcement Learning with Human Feedback (RLHF)}~\citep{christiano2017deep,stiennon2020learning,ouyang2022training} in Large Language Model (LLM) alignment. The method fine-tunes a pre-trained LLM with additional data that indicates the preference between different proposals w.r.t. customized objectives, such as safety, verbosity, coherence, \emph{etc.}~\citep{wu2024fine}.
The idea of DPO is to reparametrize the \emph{reward function} in RLHF and guide the fine-tuning process in a supervised learning manner with the preference data. 

LLM alignment also intersects with the \emph{Multi-Objective Optimization} (MOO) problem, which involves fine-tuning a model w.r.t. multiple objectives simultaneously~\citep{ji2024beavertails,wu2024fine,zhou2023beyond,rame2024rewarded}.
Production machine learning-based ranking models must strike a careful balance among competing objectives such as relevance, diversity, fairness, and other business-driven goals. This challenge is especially acute in Amazon's All Product Search (APS), where stakeholder priorities often vary across query slices, product categories, and domains. Conventional approaches that supervise ranking models using a single aggregated loss with fixed preference weights struggle to adapt to such diverse and evolving requirements. Updating these weights for each query slice typically involves retraining the model with extensive hyper-parameter optimization (HPO), a process that is both time- and resource-intensive. A common scenario involves starting from a pre-trained base ranker—trained on shared, global objectives across locales and segments—and needing to align it with additional, slice-specific objectives provided by partner teams. Efficiently fine-tuning the base model to incorporate these localized desirability signals without full retraining and without significantly detracting the model's performance on the main objectives remains a key challenge in scalable multi-objective ranking. This specific scenario is termed the \emph{Multi-Objective Fine-Tuning} (MOFT) problem.
As auxiliary objectives may conflict with each other, the notion of alignment is generalized to achieving \emph{Pareto optimality} in the MOFT setting, where the goal is to profile the \emph{Pareto front}, representing a spectrum of trade-off solutions where no single auxiliary objective can be improved without compromising another.
 
In this work, we address the MOFT task in a broad context through our proposed COS-DPO framework. This conditioned one-shot multi-objective fine-tuning framework is designed to (1) generalize DPO to the MOFT setting, (2) profile the Pareto front of the auxiliary objectives while maintaining the model performance on the main objectives with an efficient one-shot training process, and (3) offer flexible post-training controls over the trade-offs. Our codebase is publicly available at \url{https://github.com/yinuoren/cosdpo}.

\subsection{Contributions}

The main contributions of this work are as follows:
\begin{itemize}[leftmargin=1em, itemsep=0em, topsep=0em]
    \item We propose the COS-DPO method, a conditioned one-shot multi-objective fine-tuning framework that generalizes DPO to the multi-objective setting and profiles the Pareto front through one-shot training.
    \item We test our Weight-COS-DPO method across diverse tasks, including Learning-to-Rank (LTR) fine-tuning and LLM alignment, demonstrating its superior performance to achieve comprehensive Pareto fronts and its efficiency against existing baselines. 
    \item Based on our theoretical findings, we propose a novel Temperature-COS-DPO method that enhances the flexibility of post-training control over the trade-offs between the main and auxiliary objectives.
\end{itemize}

For LLM applications, we also develop a novel \emph{Hyper Prompt Tuning} design as an engineering contribution that translates the continuous vectors into a mask applied to the prefix embedding, conveying the importance weights assigned across auxiliary objectives to the LLM without altering its architecture.

\subsection{Related Works}

\paragraph{LLM Alignment.}

LLM alignment has been a popular topic in the machine learning community. RLHF has been a groundbreaking technique for alignment~\citep{christiano2017deep,schulman2017proximal,ouyang2022training,bai2022training}, which serves as a foundation for training models like GPT-4~\citep{achiam2023gpt}, and several advances have been made in this direction~\citep{dong2024rlhf,bai2022constitutional,lee2023rlaif}. To reduce computational complexity, DPO~\citep{rafailov2024direct} has been proposed as an alternative to RLHF, and further developed in the literature~\citep{pal2024smaug,wu2024beta,gheshlaghi2023general,tang2024generalized,rafailov2024r,zeng2024token,liu2024lipo,song2024preference,zhou2023beyond,guo2024controllable,yang2024rewards}. We refer readers to~\citet{shen2023large,wang2024comprehensive} for comprehensive reviews on LLM alignment.

\paragraph{Multi-Objective Optimization.}
    
MOO has been actively studied in control systems~\citep{gambier2007multi} and economics~\citep{tapia2007applications}. The main focus of the related research is the development of algorithms to profile Pareto fronts efficiently so as to understand the trade-offs between objectives. Traditional methods include the evolutionary algorithms~\citep{zhou2011multiobjective} and Bayesian optimization~\citep{laumanns2002bayesian}. Recently, gradient-based MOO methods have been studied in the machine learning settings~\citep{sener2018multi,lin2019pareto,mahapatra2020multi,liu2021stochastic,ren2024multi}. Hypernetwork-based methods are also explored by a series of works~\citep{navon2020learning,lin2020controllable,chen2022multi,hoang2023improving}.

\paragraph{Learning-to-Rank.} LTR~\citep{liu2009learning} tasks differ from traditional supervised learning in that they do not associate each sample with a simple label; instead, an optimal order of items within a group to maximize metrics, \emph{e.g.}, Normalized Discount Cumulative Gain (NDCG)~\citep{jarvelin2002cumulated, wang2013theoretical}. Typically, LTR models score documents and rank them thereby. To bridge LTR with supervised learning, various differentiable losses have been proposed as the proxy to these metrics~\citep{burges2006learning,taylor2008softrank,cao2007learning,qin2021neural, swezey2021pirank}.
In the context of Multi-Objective LTR, existing work includes label aggregation~\citep{dai2011multi, carmel2020multi}, loss aggregation~\citep{hu2018collaborative, mahapatra2023multi, mahapatra2023querywise, tang2024multi}, and hypernetwork~\citep{chen2023controllable}. 

\section{Preliminaries}

In this section, we briefly recapitulate the proximal and direct preference optimization frameworks for fine-tuning LLMs with preference data, and their generalization to listwise preference optimization with ranking data. We will also review the MOO problem in machine learning settings and then introduce the focus of this work, the Multi-Objective Fine-Tuning (MOFT) problem.

\subsection{Proximal and Direct Preference Optimization}

Suppose we have a base model $p_0(\vy|\vx)$, with $\vx$ and $\vy$ being the context and proposal, respectively, and $p_0(\vy|\vx)$ the probability of generating $\vy$ given $\vx$. The goal of DPO is to fine-tune the model $p_0(\vy|\vx)$ with preference data 
\begin{equation}
    \gD_{\mathrm{DPO}} = \{(\vx^{(k)}, \vy_1^{(k)} \succ \vy_2^{(k)})\}_{k\in[N]},
    \label{eq:dpo_dataset}
\end{equation}
where $\vy_1^{(k)} \succ \vy_2^{(k)}$ denotes $\vy_1^{(k)}$ is preferred over $\vy_2^{(k)}$ in the context of $\vx^{(k)}$. 

\paragraph{Proximal Preference Optimization.} In PPO~\citep{schulman2017proximal} or RLHF~\citep{christiano2017deep}, one first models the preference data by the \emph{Bradley-Terry-Luce (BTL) model}~\citep{bradley1952rank}:
\begin{equation}
    \P(y_1 \succ y_2 | \vx) 
    = \sigma\left(r(y_1 | \vx) - r(y_2 | \vx)\right),
\label{eq:btl}
\end{equation}
where $r(y|\vx)$ is the reward function and $\sigma$ is the sigmoid function. PPO is carried out in two steps:
(1) parametrizing $r(y|\vx)$ by a neural network $r_\phi(y|\vx)$, where parameters $\phi$ are trained by maximizing log-likelihood:
\begin{equation}
    - \Ls(r_\phi; \gD_{\mathrm{DPO}}) 
    = \E \left[\log \sigma(r_\phi(y_1 | \vx) - r_\phi(y_2 | \vx) )\right],
    \label{eq:ppo_1}
\end{equation}
and (2) fine-tuning the base model $p_0(y|\vx)$ by maximizing the expected reward while maintaining the KL divergence proximity from the base model:
\begin{equation}
    -\Ls(p_\theta; p_0, r_\phi, \beta) 
    =  \E\left[r_\phi(y | \vx) \right] - \beta \KL(p_\theta|| p_0), 
    \label{eq:ppo_2}
\end{equation}
where $\beta > 0$ is called the \emph{temperature} parameter.

\paragraph{Direct Preference Optimization.} The observation that motivates DPO~\citep{rafailov2024direct} is that the reward function $r_\phi(\vx, y)$ in~\eqref{eq:ppo_2} can be solved explicitly by letting $r_\theta(y | \vx) = \beta \log \frac{p_\theta(y|\vx)}{p_0(y|\vx)}$, and therefore, the training process can be simplified to a one-shot logistic regression:
\begin{equation}
    \begin{aligned}
        &-\Ls_{\rm DPO}(p_\theta ; p_0, \beta, \gD_{\mathrm{DPO}}) \\
        =&\E\left[ \log \sigma\left( \beta \left( \log \tfrac{p_\theta(y_1|\vx)}{p_0(y_1|\vx)} - \log \tfrac{p_\theta(y_2|\vx)}{p_0(y_2|\vx)} \right)\right)\right].
    \end{aligned}
    \label{eq:dpo}
\end{equation}
For completeness, we provide the proofs of the claims above in \cref{app:reparam}.

\subsection{Learning-to-Rank and Listwise Preference Optimization}

In LTR tasks, we are given a ranking dataset 
$$
    \gD_{\rm LTR} = \{(\vx^{(k)}, \vy_1^{(k)} \succ \cdots \succ \vy_n^{(k)})\}_{k\in[N]}, 
$$ 
where $\vy_1^{(k)} \succ \cdots \succ \vy_n^{(k)}$ denotes the ranking of the proposals in the context of $\vx^{(k)}$. As the listwise counterpart of the BTL model, the \emph{Plackett-Luce} (PL) model~\citep{plackett1975analysis} postulates that the probability of the ranking $\vpi$ is given by:
\begin{equation}
    \P(\vy_{\pi_1} \succ \cdots \succ \vy_{\pi_n} | \vx) = \prod_{i=1}^{n} \frac{e^{s(\vy_{\pi_i}| \vx)}}{\sum_{k=i}^{n} e^{s(\vy_{\pi_k}| \vx)}},
\label{eq:pl}
\end{equation}
with $s(\vy|\vx)$ being a score function, and thus the top-one probability is given by the $\softmax$ function:
$$
    \P(\vy_{i} \succ \vy_{i'},\ \forall i'\neq i| \vx) = \frac{e^{s(\vy_i| \vx)}}{\sum_{i'=1}^{n} e^{s(\vy_{i'}| \vx)}}.
$$

In many scenarios, the ranking in $\gD_{\rm LTR}$ is given by a label vector $\vz$, with $z_1 \geq \cdots \geq z_n$, indicating the preference tendency of proposals. The goal is to learn the score $s_\theta(\vy|\vx)$ parameterized by a neural network with parameters $\theta$.
One of the most popular loss functions is the \emph{ListNet} loss~\citep{cao2007learning}, which aligns an appropriate normalized version $\overline\vz$ of the labels $\vz$ with the top-one probabilities:
\begin{equation}
    -\Ls_{\rm LN}(s_\theta; \gD_{\mathrm{LTR}}) 
    = \E \left[\sum_{i=1}^{n} \overline z_{i} \log \tfrac{e^{s_\theta(\vy_i| \vx)}}{\sum_{i'=1}^{n} e^{s_\theta(\vy_{i'}| \vx)}}\right].
    \label{eq:listnet}
\end{equation}
Common choices include the $\softmax$ function for dense labels and $L_1$ normalization for sparse labels, corresponding to different modeling of the ranking data.

Similarly, the DPO framework can also be generalized from preference to ranking datasets. Suppose the base model is given in the form of a score function $s_0(\vy|\vx)$, the \emph{listwise preference optimization} (LiPO)~\citet{liu2024lipo} proposes the following loss function to obtain a fine-tuned model $s_\theta(\vy|\vx)$:
\begin{equation}
    \begin{aligned}
        & -\Ls_{\mathrm{LiPO}}(s_\theta ; s_0, \beta, \gD_{\mathrm{LTR}})\\ 
        =& \E \left[
        \sum_{i=1}^{n} \overline z_{i} \log \frac{e^{\beta (s_{\theta}(\vy_i| \vx)-s_0(\vy_i| \vx))}}{\sum_{i'=1}^{n} e^{\beta( s_{\theta}(\vy_{i'}| \vx)-s_0(\vy_{i'}| \vx))}}\right].
    \end{aligned}
    \label{eq:listnet_dpo}
\end{equation}
A justification for this loss is provided in \cref{app:reparam}. One should notice that when adopting the $L_1$ normalization, the ListNet loss~\eqref{eq:listnet} applied to the preference dataset $\gD_{\rm DPO}$ in the form of binary labels is equivalent to the DPO loss~\eqref{eq:dpo}.

\subsection{Multi-Objective Optimization}

MOO considers an optimization problem with multiple objectives $\min_{\theta \in \Theta} \bm\Ls(\theta) =  (\Ls_1(\theta), \ldots, \Ls_m(\theta))$,
where $\Theta$ is the feasible region. The goal is to profile the Pareto front, the set of trade-off solutions that cannot be improved in one objective without worsening another, or formally, the set of all $\theta$ such that for all $\theta' \in \Theta$, (1) $\Ls_i(\theta) \leq \Ls_i(\theta')$ for all $i\in[m]$, and (2) $\Ls_j(\theta) < \Ls_j(\theta')$ for some $j\in[m]$. 
This concept is motivated by the possible conflicts between objectives, and one may observe the trade-offs from the Pareto front and make informed decisions accordingly.

For many machine learning applications, the MOO problem can be formulated as follows. Given a dataset 
$$
    \gD_{\mathrm{MOO}} 
    = \{\gD_{\mathrm{MOO}}^j\}_{j\in[m]}  
    = \{\{\vy^{(k)}, z^{j, (k)}\}_{k\in[N]}\}_{j\in[m]},
$$
where $\vy^{(k)}$ is the feature vector and $z^{j, (k)}$ is the $j$-th label of the $k$-th data point, one learns a model $f_\theta(\vy)$ optimizing:
\begin{equation}
    \min_{\theta \in \Theta} \bm\Ls(f_\theta; \gD_{\mathrm{MOO}}) 
    := (\Ls_j(f_\theta; \gD_{\mathrm{MOO}}^j))_{j\in[m]},
    \label{eq:moo_ml}  
\end{equation}
where $\Ls_j(f_\theta; \gD_{\mathrm{MOO}}^j)$ is the loss function for the model $f_\theta$ w.r.t. the $j$-th objective, and the feasible region $\Theta$ is over all possible model parameters.

\subsection{Multi-Objective Fine-Tuning}

We now introduce the MOFT problem as a generalization of the LLM alignment problem to the multi-objective setting with ranking datasets, where the goal is to fine-tune an existing base model $p_0(\vy|\vx)$ (or in the form of scores $s_0(\vy|\vx)$) w.r.t. multiple \emph{auxiliary} objectives simultaneously while maintaining its performance on the \emph{main} objective(s) the base model was optimized for. Similar settings are studied by multiple concurrent works~\citep{wang2024arithmetic,mukherjee2024multi,wang2024conditional,guo2024controllable}.

In this work, we formulate the MOFT problem as follows. Given a set of item groups, each of which contains a list of items and corresponding labels w.r.t. $m$ different objectives. The dataset is of the form $\gD_{\mathrm{MOFT}} = \{\gD_{\mathrm{MOFT}}^j\}_{j\in[m]}$, with
\begin{equation}
    \gD_{\mathrm{MOFT}}^j = \left\{\vx^{(k)},(\vy_i^{(k)})_{i\in[n^{(k)}]}, (z_i^{j,(k)})_{i\in[n^{(k)}]}\right\}_{k\in[N]},
    \label{eq:ltr_dataset}
\end{equation}
where $n^{(k)}$ is the number of items, $\vx^{(k)} \in \R^D$ the context, $\vy_i^{(k)} \in \R^d$ the feature vector of the $i$-th item, and $z_i^{j,(k)} \in \R^{n^{(k)}}$ the preference tendency of the $i$-th item w.r.t. the $j$-th aspect, in the $k$-th item group.

\paragraph{Relation to LLM Alignment.} The preference dataset $\gD_{\rm DPO}$~\eqref{eq:dpo_dataset} in LLM alignment can be viewed as a special case of the MOFT problem, where $m = 1$, $n^{(k)} \equiv 2$, and the label $z_i^{1,(k)}$ is binary, being $1$ if the $i$-th item is preferred over the other, and $0$ otherwise. 

\paragraph{Relation to MOO.} MOFT is a generalization of the MOO problem~\eqref{eq:moo_ml} to the fine-tuning setting, where we aim to obtain all possible trade-offs of aligning the base model $p_0(\vy|\vx)$ (or $s_0(\vy|\vx)$) to $m$ additional datasets $\gD_{\mathrm{MOFT}}^j$.

\paragraph{Relation to LTR.} 
When $m=1$, the MOFT problem reduces to the task of fine-tuning a LTR model by viewing $\gD_{\mathrm{MOFT}}$ as an additional listwise ranking dataset. This setting will be further discussed in \cref{sec:ltr} as we apply the COS-DPO framework to this task. We refer to~\citet{liu2024lipo,song2024preference} for more discussions on LLM alignment with listwise data.

We thus aim to design a framework as a versatile solution to the MOFT problem that can not only address the LTR fine-tuning, LLM alignment, and MOO tasks simultaneously but also synergize the state-of-the-art practices in each of these areas to achieve the best performance.

\section{Methodology}

In this section, we present the COS-DPO framework, a conditioned one-shot multi-objective fine-tuning framework that generalizes the DPO framework for LLM alignment to the MOFT setting and profiles the Pareto front. Below, we consider the following MOFT problem:
\begin{equation}
    \min_{\theta \in \Theta} \bm\Ls_{\rm LiPO}(s_\theta ; s_0, \vbeta, \gD_{\mathrm{MOFT}}),
    \label{eq:moft_listnet}
\end{equation}
where the vector of the loss functions $\bm\Ls_{\rm LiPO}$ consists of the LiPO loss functions $(\Ls_{\rm LiPO}(s_\theta; s_0, \beta_j, \gD_{\mathrm{MOFT}}^j))_{j\in[m]}$~\eqref{eq:listnet_dpo} for each auxiliary objective, and $\vbeta = (\beta_1, \ldots, \beta_m)$ is the vector of temperatures that control the trade-off between the main objective and each auxiliary objective.

\paragraph{Linear Scalarization.} The most straightforward way to solve this MOO problem is to train the model $s_\theta$ with a linear combination of the preference data~\citep{zhou2023beyond}:
\begin{equation*}
    \Ls_{\vw}(s_\theta; s_0, \vbeta, \gD_{\mathrm{MOFT}}) 
    = \vw^\top \bm\Ls_{\rm LiPO}(s_\theta ; s_0, \vbeta, \gD_{\mathrm{MOFT}}),
\end{equation*}
where $\vw = (w_1, \ldots, w_m)^\top \in \Delta^m$ is the weight vector that reflects the importance we assign over the objectives, and with $\Delta^m$ being the $m$-dimensional probability simplex. 

As $\vw$ iterates over $\Delta^m$, the model $s_\theta$ will be optimized over a specific trade-off between the main objective and the auxiliary objectives and possibly land on the Pareto front. This approach is able to obtain the complete Pareto front when it is convex~\citep{jakob2014pareto}.

\paragraph{Conditioned One-Shot Networks.}

\begin{figure*}
    \centering
    \begin{subfigure}[b]{.46\textwidth} 
        \centering
        \begin{tikzpicture}
            \foreach \x in {0,1,...,5} {
            \draw[gray!40, opacity=0.5, thin] (\x,0) -- (\x,5);
            }
            \foreach \y in {0,1,...,4} {
            \draw[gray!40, opacity=0.5, thin] (0,\y) -- (6,\y);
            }
            \node (x) at (6, 0) {}; 
            \node at (5.7, 0.4) {Aux. Obj. 1};
            \node (y) at (0, 5) {};
            \node at (0.7, 5) {Aux. Obj. 2};
        
            \draw[->, thick] (0,0) -- (x);
            \draw[->, thick] (0,0) -- (y);
        
            \draw (1.8, 4) to[bend left=40] node[above, sloped] {PF at $\vbeta$} (5.04,1.5);
            \draw (1.2, 3.5) to[bend left=40] node[above, sloped] {PF at $2\vbeta$} (4.44,1);
            \draw (2.4, 4.5) to[bend left=40] node[above, sloped] {PF at $\vbeta/2$} (5.64,2);
        
            \draw[<->, dashed] (4.44, .8) node[below] {\footnotesize larger} -- node[below, sloped] {$\vbeta$} (6, 2) node[right] {\footnotesize smaller};
        
            \draw[<->, dashed] (1.44, 3) node[left] {\footnotesize $(0,1)$} -- node[above, sloped] {$\vw$} (3.84, 1.1) node[below] {\footnotesize $(1,0)$};
        \end{tikzpicture}
        \caption{Weight-COS-DPO}
        \label{fig:weight_cos}
        \end{subfigure}
        \hspace{-1em}
        \begin{subfigure}[b]{.46\textwidth} 
        \centering
        \begin{tikzpicture}
            \foreach \x in {0,1,...,5} {
            \draw[gray!40, opacity=0.5, thin] (\x,0) -- (\x,5);
            }
            \foreach \y in {0,1,...,4} {
            \draw[gray!40, opacity=0.5, thin] (0,\y) -- (6,\y);
            }
            \node (x) at (6, 0) {}; 
            \node at (5.7, 0.4) {Aux. Obj. 1};
            \node (y) at (0, 5) {};
            \node at (0.7, 5) {Aux. Obj. 2};
        
            \draw[->, thick] (0,0) -- (x);
            \draw[->, thick] (0,0) -- (y);
        
            \draw (2.4, 4.5) to[bend left=40] node[above, sloped] {PF at $(\beta_1, \beta_2)$} (5.64, 1.8);
            \draw (2.28, 4.4) to[bend left=30] node[midway, xshift=-65pt, yshift=25pt] {PF at $(\beta_1, 2\beta_2)$} (4.44,.7);
            \draw (1.2, 3.4) to[bend left=20] node[midway, xshift=57pt, yshift=-55pt] {PF at $(2\beta_1, \beta_2)$} (5.64,1.7);
            \draw (1.2, 3.3) to[bend left=30] node[sloped, below, xshift=-2pt, yshift=-7pt] {PF at $(2\beta_1, 2\beta_2)$} (4.32,.7);
        \end{tikzpicture}
        \caption{Temperature-COS-DPO}
        \label{fig:temperature_cos}
        \end{subfigure}
        \caption{Conceptual Illustration of Post-Training Controls in the COS-DPO Framework with 2 auxiliary objectives.}
        \label{fig:post_training}
    \end{figure*}

An efficient way to profile the Pareto front of this MOFT problem is to use \emph{hypernetworks}~\citep{navon2020learning,hoang2023improving}, \emph{i.e.}, additional neural networks that generate the model parameters according to the weight vector $\vw$.
As an efficient and robust alternative to hypernetworks,~\citet{ruchte2021scalable} proposes \emph{conditioned one-shot networks} that directly input the weight vector $\vw$ to the model, with successful applications to multiple MOO tasks.

\subsection{Weight-Conditioned Networks}

COS-DPO generalizes the idea of conditioned one-shot networks to the MOFT setting. To be specific, we propose \emph{Weight-Conditioned One-Shot} (Weight-COS) networks $s_{\theta}(\cdot, \vw | \vx)$ that not only take in the data but also condition on the weight $\vw$ over objectives. Intuitively, it formulates the MOFT problem into a ``meta-learning'' problem, and the model $s_\theta(\cdot, \vw | \vx)$ is trained to optimize the objectives over a distribution of weight vectors $\vw$. 

Since $\vw$ is supported on $\Delta^m$, we sample $\vw$ from a Dirichlet distribution $\mathrm{Dir}(\valpha)$ during each epoch of the training process, where $\valpha$ is the concentration parameter. The Weight-COS-DPO method is equivalent to optimize the Weight-COS network $s_{\theta}(\cdot, \vw | \vx)$ w.r.t. the following loss function:
\begin{equation}
    \begin{aligned}
        &\Ls_{\mathrm{W\hyphen COS}}(s_\theta; s_0, \vbeta, \gD_{\mathrm{MOFT}}, \valpha)\\ 
        = &\E_{\vw \sim \mathrm{Dir}(\valpha)}[\Ls_{\vw}(s_\theta(\cdot, \vw | \vx); s_0, \vbeta, \gD_{\mathrm{MOFT}})].
    \end{aligned}
    \label{eq:W-COS_loss}
\end{equation}
Ideally, \emph{i.e.}, when the model has sufficient capacity and the Pareto front is smooth and convex, the optimized model $s_{\theta, \vbeta}(\cdot, \vw | \vx)$ would be Pareto optimal w.r.t. the auxiliary objectives for each $\vw \in \Delta^m$ and thus form the Pareto front.

The Weight-COS-DPO method is summarized in \cref{alg:COS-DPO}.

\IncMargin{1.5em}
\begin{algorithm}[ht]
    \caption{Weight-COS-DPO}
    \label{alg:COS-DPO}
    \Indm
    \KwData{Base model $s_0(\vy | \vx)$, dataset $\gD_{\mathrm{MOFT}}$ concentration parameter $\valpha$, temperature $\vbeta$.}
    \KwResult{Fine-Tuned model $s_{\theta, \vbeta}(\cdot, \cdot | \vx)$.}
    \Indp
    \For{$e = 1$ \KwTo $N_{\rm steps}$}{
        Sample $\vw' \sim \mathrm{Dir}(\valpha)$\;
        $\theta \leftarrow \theta - \eta \nabla_{\theta} \Ls_{\vw}(s_\theta(\cdot, \vw' | \vx); s_0, \vbeta, \gD_{\mathrm{MOFT}});$
    }
\end{algorithm}
\DecMargin{1.5em}

\subsection{Linear Transformation Property}
\label{sec:linear}

Due to the linearity of the Weight-COS-DPO method, we underscore the following linear transformation property:
\begin{proposition}[Linear Transformation Property]
    For any $\vbeta \in \R_+^m$ and $\vw\in\Delta^m$, we denote the model obtained by optimizing the Weight-COS-DPO loss~\eqref{eq:W-COS_loss} with temperature $\vbeta$ as $s_{\theta, \vbeta}(\vy, \vw | \vx)$. 
    
    Then $s_{\theta, \vbeta}(\vy, \vw | \vx)$ should satisfy the following \emph{linear transformation property} that for any $c > 0$, we have
    \begin{equation}
        s_{\theta, c\vbeta}(\vy, \vw | \vx) = \left(1 - \tfrac{1}{c}\right) s_0(\vy | \vx) + \tfrac{1}{c} s_{\theta, \vbeta}(\vy, \vw | \vx)
        \label{eq:linear_transform}
    \end{equation}
    is also an optimal solution to the Weight-COS-DPO loss~\eqref{eq:W-COS_loss} with temperature $c\vbeta$.
    \label{prop:linear}
\end{proposition}

The proof of this proposition is provided in \cref{app:linear} and will be empirically validated with experiments as shown in \cref{fig:from}. Powered by \cref{prop:linear}, Weight-COS networks offer post-training controls over the trade-offs between the main and auxiliary objectives. To be specific, once we have trained a model $s_{\theta}(\vy, \vw | \vx)$ with a specific temperature $\vbeta$, we may also obtain the Pareto front under a different temperature $c\vbeta$ by simply scaling the output as: 
\begin{equation}
    s_{\theta}(\vy, \vw | \vx)\leftarrow\left(1 - \tfrac{1}{c}\right) s_0(\vy | \vx) + \tfrac{1}{c} s_{\theta}(\vy, \vw | \vx).
    \label{eq:linear_transform_4}
\end{equation}

Consequently, as illustrated in \cref{fig:weight_cos}, after only one training process with a specific choice of temperature $\vbeta$, Weight-COS-POS allows post-training controls over two kinds of trade-offs: (1) those between the auxiliary objectives by adjusting the weight vector $\vw$, and (2) those between the fidelity to the base model and its performance on the auxiliary objectives by scaling temperature $\vbeta$ with~\eqref{eq:linear_transform_4}.  

\subsection{Temperature-Conditioned Networks}

Although \cref{prop:linear} provides a flexible way to control the trade-offs between the main and auxiliary objectives for the Weight-COS-DPO method, it covers only one specific degree of freedom in the space of the space of the second type of trade-offs.
Generally speaking, the model should exhibit different Pareto fronts for different temperature parameters $\vbeta \in \R^m_+$, and thus one may consider using different temperatures across objectives and also a disproportionate post-training scaling of $\vbeta$ to achieve more flexible control over the Pareto front (\emph{cf.} \cref{fig:temperature_cos} vs. \cref{fig:weight_cos}).

To this end, as a generalized version of the Weight-COS networks, we propose the \emph{Temperature-Conditioned One-Shot} (Temperature-COS) networks by further incorporating the temperature parameter $\vbeta$ into the model input in a similar manner as the weight vector $\vw$, denoted by $s_{\theta}(\vy, \vw, \vbeta | \vx)$. \cref{prop:linear} implies that the temperature $\vbeta \in \R^m_+$ is actually of $m-1$ degrees of freedom, and thus we propose to use the following reparametrization by projecting $\vbeta$ to its $L^1$-normalization $\overline \vbeta = \frac{\vbeta}{\|\vbeta\|_1} \in \Delta^{m}$, \emph{i.e.},
\begin{equation}
    \begin{aligned}
        &s_{\theta}(\vy, \vw, \vbeta | \vx) \\
        =& \big(1 - \tfrac{1}{\|\vbeta\|_1}\big) s_0(\vy|\vx) + \tfrac{1}{\|\vbeta\|_1} s_{\theta}(\vy, \vw, \overline\vbeta|\vx).
    \end{aligned}
    \label{eq:temperature_reparam}
\end{equation}

The training is thus conducted similarly to the Weight-COS-DPO method by randomly sampling $\vbeta \in \R_+^m$ over a certain distribution $\gD_\vbeta$ valued in $\R^m_+$ besides $\vw$ at each epoch. The loss of Temperature-COS-DPO can be written as 
\begin{equation}
    \begin{aligned}
        &\Ls_{\mathrm{T\hyphen COS}}(s_\theta; s_0, \gD_{\mathrm{MOFT}}, \valpha)\\ 
        = &\E_{\vbeta\sim\gD_\vbeta, \vw \sim \mathrm{Dir}(\valpha)}[\Ls_{\vw}(s_\theta(\cdot, \vw, \vbeta | \vx); s_0, \gD_{\mathrm{MOFT}}) ].
    \end{aligned}
    \label{eq:T-COS_loss}
\end{equation}
The Temperature-COS-DPO method is provided in \cref{alg:temp_COS-DPO}. 

\IncMargin{1.5em}
\begin{algorithm}[ht]
    \caption{Temperature-COS-DPO}
    \label{alg:temp_COS-DPO}
    \Indm
    \KwData{Base model $s_0(\vy | \vx)$, dataset $\gD_{\mathrm{MOFT}}$, concentration parameter $\valpha$, temperature distribution $\gD_\vbeta$.}
    \KwResult{Fine-Tuned Model $s_{\theta}(\cdot, \cdot, \cdot | \vx)$.}
    \Indp
    \For{$e = 1$ \KwTo $N_{\rm steps}$}{
        Sample $\vw' \sim \mathrm{Dir}(\valpha)$, $\vbeta' \sim \gD_\vbeta$\;
        $\theta \leftarrow \theta - \eta \nabla_{\theta} \Ls_{\vw}(s_\theta(\cdot, \vw', \vbeta' | \vx); s_0, \gD_{\mathrm{MOFT}});$
    }
\end{algorithm}
\DecMargin{1.5em}

We remark that both Weight-COS-DPO and Temperature-COS-DPO can be implemented with penalization terms to foster the exploration of the Pareto front without affecting the validity of the linear transformation property (\emph{cf.} \cref{prop:linear}). We refer to \cref{app:penalization} for more details.

\begin{figure*}[ht]
    \centering
    \begin{subfigure}[t]{.45\textwidth}
        \centering
        \includegraphics[width=\linewidth]{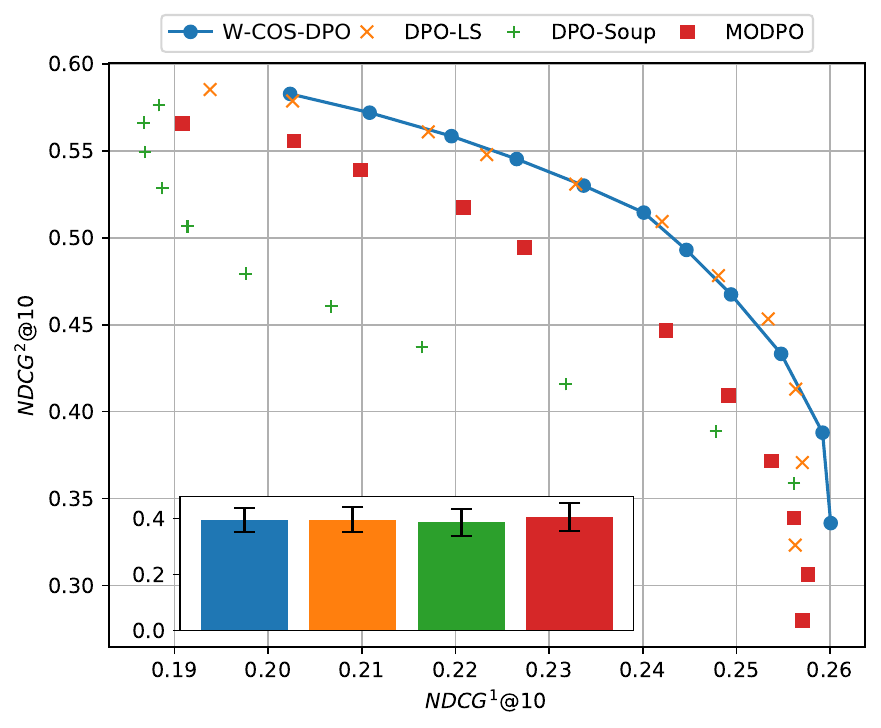}
        \caption{Objective I vs Objective II.}
        \label{fig:ltr_2_a}
    \end{subfigure}
    \hspace{.5em}
    \begin{subfigure}[t]{.45\textwidth}
        \centering
        \includegraphics[width=\linewidth]{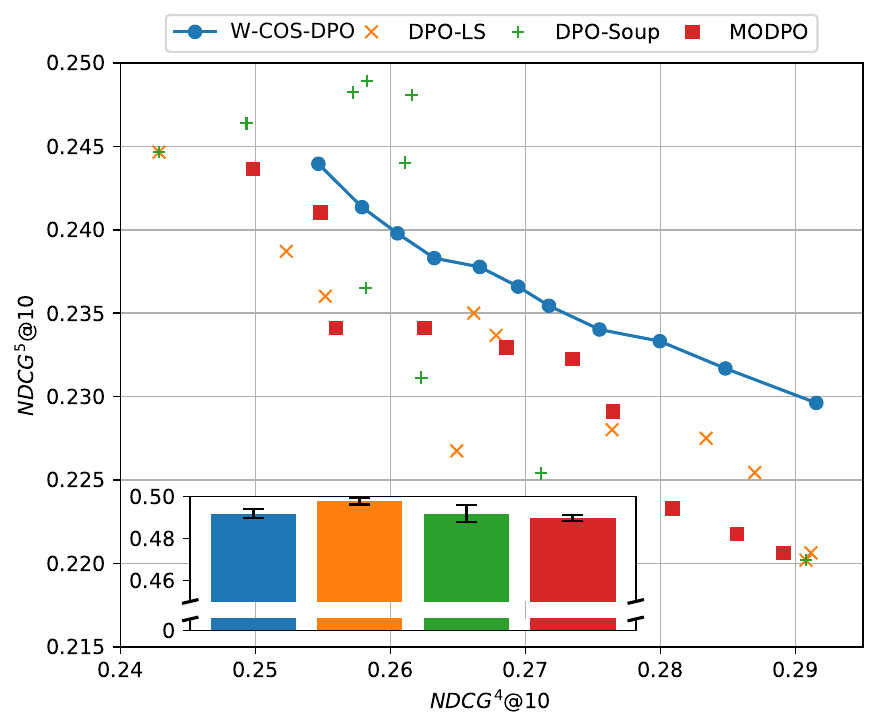}
        \caption{Objective IV vs Objective V.}
        \label{fig:ltr_2_b}
    \end{subfigure}
    \caption{Comparison of Pareto fronts obtained by W-COS-DPO and baselines on the MSLR-WEB10K dataset with 2 auxiliary objectives. Two axes denote the NDCG@10 of the two auxiliary objectives (the higher, the better). The inset plot shows the average NDCG@10 of the main objective, with the error bar denoting the standard deviation across the 11 sampled points.
    }
    \label{fig:ltr_2}
\end{figure*}
\begin{figure*}[ht]
    \centering
    \begin{subfigure}{.43\textwidth}
        \centering
        \includegraphics[width=\linewidth]{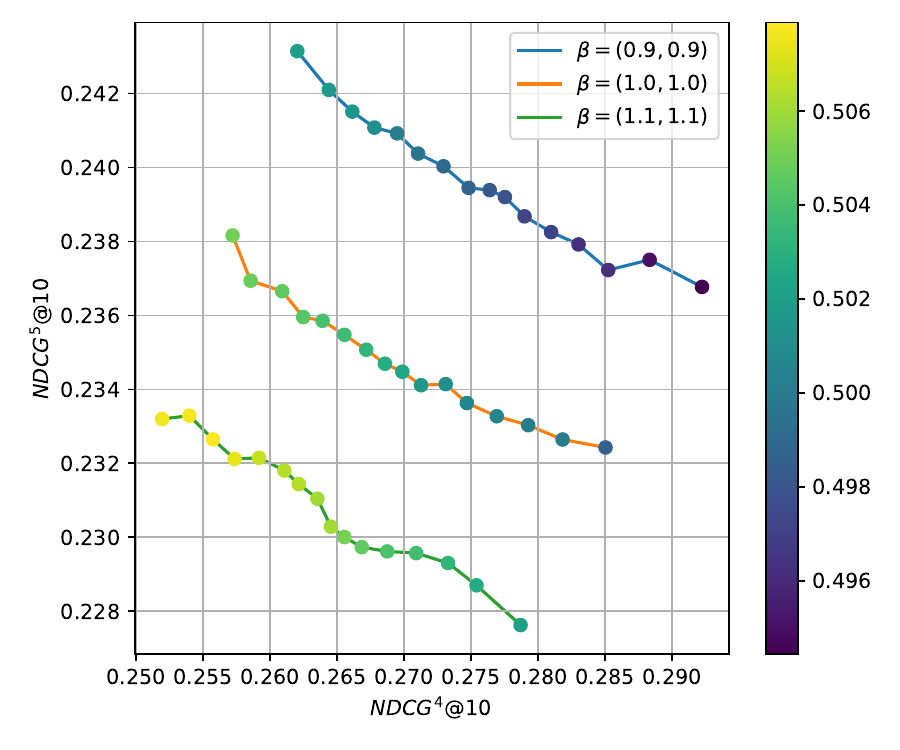}
        \caption{W-COS-DPO.}
        \label{fig:ex_weight}
    \end{subfigure}
    \hspace{1em}
    \begin{subfigure}{.43\textwidth}
        \centering
        \includegraphics[width=\linewidth]{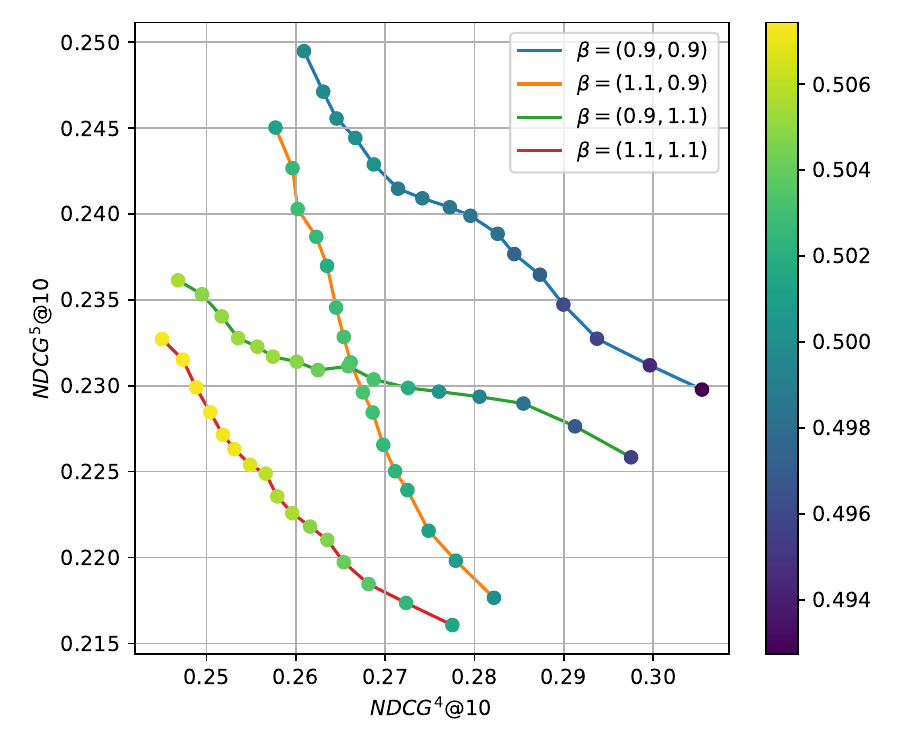}
        \caption{T-COS-DPO.}
        \label{fig:ex_temp}
    \end{subfigure}
    \caption{Examples of post-training control over temperature $\vbeta$ of both W-COS-DPO and T-COS-DPO.}
    \label{fig:ex_post_training}
\end{figure*}

\section{Experiments}

In this section, we provide the detailed experiment design and results of the COS-DPO framework for different applications, including the LTR fine-tuning and LLM alignment task. We compare the Weight-COS-DPO (W-COS-DPO in below) method with the following existing baselines, including the DPO Linear Scalarization (DPO-LS) method, the DPO Soup method~\citep{rame2024rewarded}, and the MO-DPO method~\citep{zhou2023beyond}. For details and further discussion of these baselines, we refer to \cref{app:baselines}. For each baseline, we use the same number of weight vectors $\vw$ for a fair comparison.  
The \emph{Hypervolume (HV)} indicator~\citep{zitzler2004indicator} is adopted for evaluating the performance of MOFT methods (\emph{cf.} \cref{app:hypervolume}). We also carry out preliminary experiments on the proposed Temperature-COS-DPO (T-COS-DPO in below) method on the LTR fine-tuning task to demonstrate its feasibility.

\subsection{Learning-to-Rank Fine-Tuning}
\label{sec:ltr}

We first test the COS-DPO framework on the task of fine-tuning LTR models. We adopt the Microsoft Learning-to-Rank Web Search (MSLR-WEB10K) dataset~\citep{qin2013introducing} for the LTR task.
The MSLR-WEB10K dataset consists of 10,000 groups ($N = 10^4$), with 5 auxiliary objectives ($m = 5$). We refer to \cref{app:datasets} for more details.

In MSLR-WEB10K dataset, the information $\vx$ has been incorporated into the feature vectors $\vy$ by upstream data processing. We thus use a 2-layer transformer architecture of hidden dimension 128 for the base model $s_0(\vy)$, and the model $s_\theta(\cdot, \vw)$ is designed as a 2-layer transformer architecture of hidden dimension 64 with $\vw$ concatenated to the input of the first layer.

We first apply W-COS-DPO to the case $m=2$ for better visualization. \cref{fig:ltr_2_a} presents the Pareto front of two sparse labels ($\overline \vz = \vz/\|\vz\|_1$ in~\eqref{eq:listnet_dpo}) with a convex Pareto front, while \cref{fig:ltr_2_b} presents that of two dense labels ($\overline \vz = \mathrm{softmax}(\vz)$ in~\eqref{eq:listnet_dpo}) with an ill-posed Pareto front. 
W-COS-DPO obtains comprehensive Pareto fronts that dominate those of the baselines in both pairs. Notably, our method obtains a smooth Pareto front in \cref{fig:ltr_2_b} while other baselines fail.
With a common temperature parameter $\vbeta$ used across all methods, the inset plots demonstrate that the superior performance of our method is not at the cost of the main objective, as the NDCG@10 of the main objective is comparable or even slightly better to baselines.

\begin{table}[h]
    \small
    \centering
    \setlength{\tabcolsep}{2pt} 
    \begin{tabular}{c|c c c c}
        \toprule
        \bf Metric  & \bf DPO-LS & \bf DPO Soup & \bf MO-DPO & \bf W-COS-DPO \\
        \midrule
        Aux. HV & 1.648e-3 & 1.468e-3 & 1.263e-3 & \bf 2.039e-3 \\
        Avg. M. Scr.  & 0.355  & 0.382  & 0.360  & \bf 0.432  \\
        (w/$\pm$Std)  & ($\pm$ 0.029) & ($\pm$ 0.032) & ($\pm$ 0.024) & ($\pm$ 0.028) \\
        Time (s) & 14649.15 & 6061.69 & 27059.70 & \bf 4043.47 \\
        \# Params. & 551,232 & 250,615 & 801,792 & \bf 50,432 \\
        \bottomrule
    \end{tabular}
    \caption{HV metric and training time of COS-DPO and the baselines on the MSLR-WEB10K dataset with 5 auxiliary objectives. The reference point is set to $(0,0)$. 11 points are produced for computing HV. Avg. M. Scr. (w/$\pm$Std) refers to the average NDCG@10 (with standard deviation) of the main objective across the 11 points.}
    \label{tab:ltr_5}
\end{table}

We also test W-COS-DPO on a more complicated case where we have 5 auxiliary objectives ($m=5$), as shown in \cref{tab:ltr_5}. Our results demonstrate W-COS-DPO achieves a higher hypervolume metric with significantly less training time and number of parameters compared to the baselines and comparably good preservation of the performance on the main objective (see also \cref{app:time}).
While the computational cost of DPO-LS and MO-DPO, grows exponentially with the number of objectives, our method maintains a linear growth with almost intact performance, indicating the efficiency of our W-COS-DPO method in handling high-dimensional MOFT problems in the LTR task.

\begin{figure*}[ht]
    \centering
    \begin{tikzpicture}[scale=0.9, transform shape]
        \node[rectangle, draw, fit={(-.6,0) (-.5,1.35)}] (input) {};
        \node[above of=input, yshift=2] {Input};
        \node[left of=input, xshift=12, align=center] {\color{gray} \it How\\ \color{gray} \it do\\ \color{gray} \it I\\ \color{gray} \it $\vdots$};

        \node[rectangle, rounded corners=.5cm, draw, fit={(.8,0) (2.9,1.35)}, fill=yellow!70!white, fill opacity=0.3, text opacity=1] (embedding) {Token\\ Embedding};
        \node[circle, draw, plus, minimum size=.45cm, inner sep=0pt] (oplus) at (4.2, .65) {};
        \node[circle, draw, sin wave, minimum size=.45cm, inner sep=0pt] (osin) at (4.2, 1.3) {};
        \node[above of=osin, yshift=-8, xshift=-3, align=center] (pos encoding) {Positional\\ Encoding};

        \node[rectangle, draw, fit={(1.8,-2.5) (4.4,-.8)}, dashed] (transformer) {};
        \node[below of=transformer, yshift=-7] {$\times n$};
        \node[rectangle, rounded corners=.5cm, draw, fit={(2,-2.3) (4.2,-1)}, fill=yellow!70!white, fill opacity=0.3, text opacity=1] (transformer block) {Transformer Blocks};

        \node[rectangle, draw, fit={(0,-2.75) (1.2,-.55)}] (output) {};
        \node[rectangle, draw, fit={(0,-1) (1.2,-.55)}, fill=gray, fill opacity=0.8, text opacity=1] (output prefix) {};
        \node at (.6, -2.0) {Output};

        \node[rectangle, rounded corners=.5cm, draw, fit={(-1.8,-2.7) (-.6,-1.6)}, fill=yellow!70!white, fill opacity=0.3, text opacity=1] (lm head) {LM Head};
        \node[above of=lm head, yshift=18, xshift=-5] (discard) {(Discard)};

        \node[rectangle, draw, fit={(5.4,-1.6) (6.6,.6)}] (embed) {};
        \node[above of=embed, yshift=13] {Embedding};

        \draw[->] (input) -- (embedding);
        \draw[->] (embedding) -- (oplus);
        \draw[->] (osin) -- (oplus);
        \draw[->] (oplus) to[out=-45, in=170] (embed);
        \draw[->] (embed) to[out=-130, in=0] (transformer block);
        \draw[->] (transformer block)--(output);
        \draw[->] (output) -- (lm head);
        \draw[->] (output prefix) -- (discard);

        \node[rectangle, draw, fit={(5.4,0.2) (6.6,.6)}, fill=violet!80!white, fill opacity=0.3, text opacity=1, label=center:Prefix] (prefix) {};
        \node[left of=prefix, xshift=4] {$k$};

        \node[rectangle, draw, fit={(7.8,1.1) (9,1.5)}, fill=red!70!white, fill opacity=0.3, text opacity=1] (prompt) {};
        \node[right of=prompt, xshift=.6cm, align=left] {Prefix\\ Embedding};
        \node[left of=prompt, xshift=4] {$k$};

        \node[rectangle, draw, fit={(9.6,0.2) (10.8,.6)}, fill=blue!50!white, fill opacity=0.3, text opacity=1] (mask) {};
        \node[right of=mask, xshift=.2cm, align=left] {Mask};

        \node[circle, draw, minimum size=.45cm, inner sep=0pt] (otimes) at (8.4, .4) {$\bigodot$};
        
        \draw[->] (prompt) -- (otimes);
        \draw[->] (mask) -- (otimes);
        \draw[->] (otimes) -- (prefix);

        \node[rectangle, draw, fit={(9,-.9) (9.2,-.5)}, fill=blue!70!white, fill opacity=0.3, text opacity=1] (mask1) {};
        \node[above of=mask1, yshift=-14] {$r$};
        \node[left of=mask1, xshift=16] {$k$};
        \node[rectangle, draw, fit={(11.4,-.8) (12.6,-.6)}, fill=blue!70!white, fill opacity=0.3, text opacity=1] (mask2) {};
        \node[right of=mask2, xshift=-4] {$r$};
        \node[circle, draw, cross, minimum size=.45cm, inner sep=0pt] (otimes2) at (10.2, -0.7) {};
        \draw[->] (mask1) -- (otimes2);
        \draw[->] (mask2) -- (otimes2);
        \draw[->] (otimes2) -- (mask);

        \node[rectangle, draw, fit={(10,-2.5) (11.4,-2.4)}, label=center:{$\vw$}] (w) {};
        \node[left of=w, xshift=-13, yshift=0, align=right] {Weight};
        \node[below of=w, yshift=11, xshift=-15, align=left] {\color{gray} \it 50\% Quality, 30\% Safety, 10\% Verbosity, $\cdots$};
        \node[right of=w, xshift=25, align=left] {\color{gray} $(.5, .3, .1, \cdots)$};

        \node[rectangle, rounded corners, draw, fit={(8.5,-1.8) (9.7,-1.4)}, label=center:MLP, fill=yellow!70!white, fill opacity=0.3, text opacity=1] (mlp1) {};
        \node[rectangle, rounded corners, draw, fit={(11.4,-1.8) (12.6,-1.4)}, label=center:MLP, fill=yellow!70!white, fill opacity=0.3, text opacity=1] (mlp2) {};

        \draw[->] (w) -- (mlp1);
        \draw[->] (w) -- (mlp2);
        \draw[->] (mlp1) -- (mask1);
        \draw[->] (mlp2) -- (mask2);

        \draw[dashed] (13.8, -3.3)--(13.8, 1.8)--(5, 1.8)--(5, -.1)--(7.8, -.1)--(7.8, -2)--(6.8, -2)--(6.8, -3.3)--cycle;
        \node at (9.2, 2.1) {\bf Hyper Prompt Tuning};
    \end{tikzpicture}
    \caption{Illustration of the implementation of the W-COS-DPO framework for the LLM alignment task. The proposed {\bf Hyper Prompt Tuning} method, highlighted within the dashed box, transforms the weight vector $\vw$ into a mask and passes it to the LLM via prompt tuning. $k$ denotes the number of virtual tokens, and $r$ is the rank of the weight mask.}
    \label{fig:ICOS_llm}
\end{figure*}

We present an example of the post-training control over the trade-offs between the main and auxiliary objectives by both W-COS-DPO and T-COS-DPO in \cref{fig:ex_post_training}. Detailed experimental settings are deferred to \cref{app:post_training,app:temp}. While the model obtained by W-COS-DPO in \cref{fig:ex_weight} is only able to apply linear translation to the Pareto front, the model obtained by the T-COS-DPO method in \cref{fig:ex_temp} demonstrates its capability of scaling the Pareto front in an objective-specific manner. This is in exact accordance with their expected performance as illustrated in \cref{fig:post_training} and validates the feasibility of the proposed T-COS-DPO method. Our theoretical findings (\cref{prop:linear}) are also confirmed by \cref{fig:temp}, where similar Pareto fronts are obtained by training with different temperatures and scaling the output with~\eqref{eq:linear_transform_4}.

We study the hyperparameter robustness of W-COS-DPO on the LTR fine-tuning task in \cref{app:ablation}. Specifically, we evaluate its sensitivity to the concentration parameter $\valpha$ (\emph{cf.} \cref{app:alpha}) and the model depth (capacity) (\emph{cf.} \cref{app:depth}). Furthermore, we introduce and compare the performance of two different parametrizations of $s_{\theta}(\cdot, \vw | \vx)$ in \cref{app:ppodpo}, namely (a) \emph{Training-from-Scratch} and (b) \emph{Augmentation Network}, which exhibit different trade-offs between the performance and the computational cost and thus may serve different purposes in practice.

\subsection{LLM Alignment Task}
\label{sec:llm}

We then apply the COS-DPO framework to the LLM alignment task. The PKU-SafeRLHF dataset~\citep{ji2024pku} is adopted for experiments, which consists of 83.4k entries, with 2 auxiliary objectives ($m = 2$). We refer to \cref{app:datasets} for more details on the dataset.

In contrast to the LTR task, where we directly concatenate the weight $\vw$ to the input of the W-COS network $s_{\theta}(\cdot, \vw | \vx)$. In LLM alignment, this strategy is generally infeasible, since the input of transformers is tokenized prompts, and the weight vector $\vw$ is a real vector that is not considered in tokenization and thus a direct concatenation may cause the model to fail to generate meaningful responses. To address this issue and incorporate the information of the weight vector $\vw$ into the LLM with the least modification to the model and the training process, we propose a novel design, called \textit{Hyper Prompt Tuning (HPT)}. 

The mechanism of HPT is shown in \cref{fig:ICOS_llm}. Inspired by Prompt Tuning~\citep{lester2021power}, HPT augments the input embedding obtained post token embedding and positional encoding with a trainable prefix embedding block that is controlled by the weight vector $\vw$. Specifically, HPT follows the following steps: (1) HPT takes in a weight vector $\vw \in \Delta^m$ that indicates the importance across auxiliary objectives and, through simple trainable MLPs, produces two matrices, the matrix product of which forms the mask, (2) the mask is multiplied entrywise with a trainable prefix embedding block with $k$ virtual tokens, and (3) the prefix embedding block is then concatenated to the input embedding as a prefix and fed into the transformer blocks of the LLM. In contrast to Multi-Task Prompt Tuning~\citep{wang2023multitask}, which permits only a finite number of tasks, one can pass a continuum of weight by HPT into the LLM, offering both flexibility and versatility. 

We thus perform fine-tuning to the GPT-2 model~\citep{radford2019language} and the Alpaca-7B-Reproduced model~\citep{dai2023safe}, following the practice by~\citet{zhou2023beyond} via Parameter-Efficient Fine-Tining (PEFT) with $\alpha = 8$ and $r = 4$ in the low-rank adaptions (LoRA) to the modules within the model. 
For W-COS-DPO, we adopt the Hyper Prompt Tuning technique with $k = 8$ and $r = 4$.
To ensure a fair comparison, baseline methods will also be augmented with the prompt tuning of $k = 8$ on top of LoRA. Our method is built upon the TRL package~\citep{vonwerra2022trl}, and the implementation of the HPT is compatible with the PEFT package~\citep{peft}, allowing easy integration with existing LLMs.
All the experiments are conducted on 8$\times$ NVIDIA A100 GPUs. 

\begin{table}[h]
    \centering
    \setlength{\tabcolsep}{2.5pt}
    \begin{tabular}{c|c c |c c}
        \toprule
        \multirow{2}{*}{\bf Method} &\multicolumn{2}{c|}{\bf GPT-2} &\multicolumn{2}{c}{\bf Alpaca-7B-Rep.}  \\[.5ex]
         &  \bf HV & \bf Time (s) &  \bf HV& \bf Time (s) \\
        \midrule
        DPO-LS & 0.1767 & 15148.5 & 0.1687 & 94156.1\\
        DPO Soup & 0.1840 & 2755.5 & 0.1427 & 17138.7\\
        \bf W-COS-DPO & \bf 0.1942 & \bf 1396.8 & \bf 0.1689 & \bf 8520.2\\
        \bottomrule
    \end{tabular}
    \caption{HV metric and training time of W-COS-DPO and the baselines on the PKU-SafeRLHF dataset. The reference point is set to $(1.1,1.1)$, and 11 points are produced for computing HV.}
    \label{tab:llm}
\end{table}

\begin{figure*}[ht]
    \centering
    \begin{subfigure}[b]{.45\textwidth}
        \includegraphics[width=\linewidth]{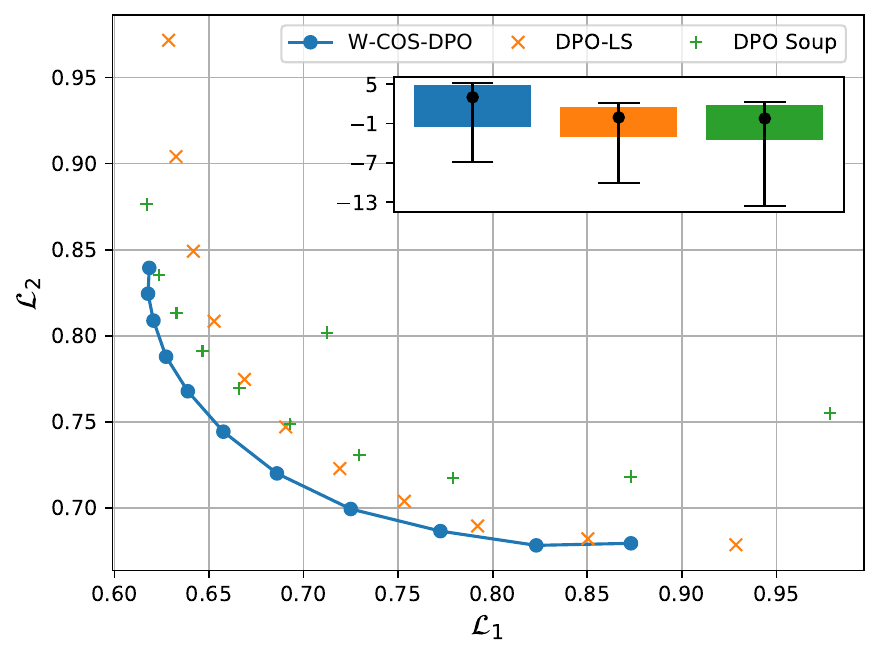}
        \caption{GPT-2}
        \label{fig:llm_gpt2}
    \end{subfigure}
    \hspace{1em}
    \begin{subfigure}[b]{.45\textwidth}
        \includegraphics[width=\linewidth]{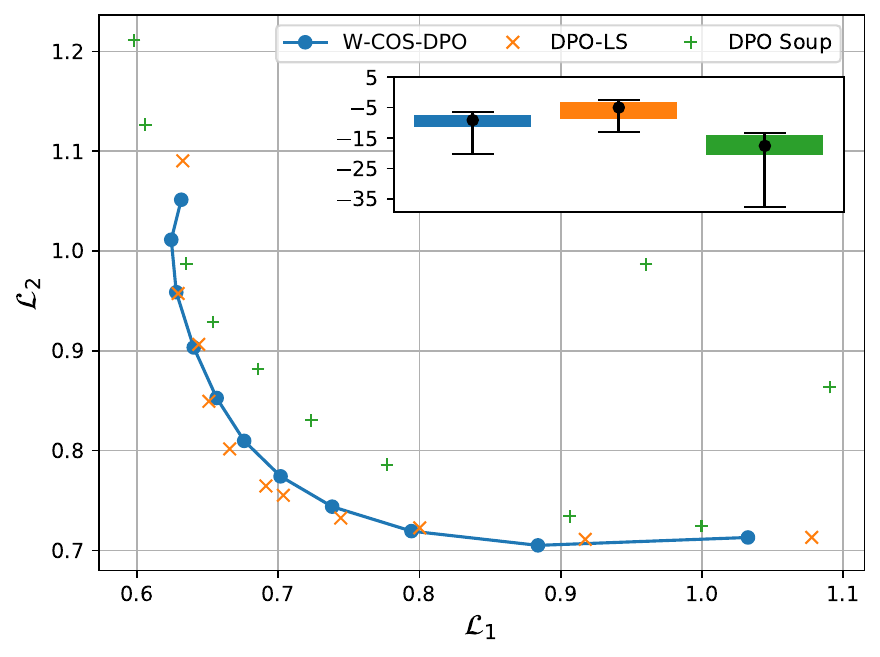}
        \caption{Alpaca-7B-Reproduced}
        \label{fig:llm_llama2}
    \end{subfigure}
    \caption{Comparison of Pareto fronts obtained by W-COS-DPO and baselines on the PKU-SafeRLHF dataset. Two axes denote the expected cross-entropy error of two auxiliary objectives (the lower, the better). The inset plot shows the interquartile range (IQR) of the log-likelihood deviation of the response from the reference model across the test dataset.}
    \label{fig:llm}
\end{figure*}

In this task, we compare the results of W-COS-DPO with those of DPO-LS and DPO Soup and we refer readers to discussions in \cref{app:baselines} for the comparison with MO-DPO.
For all experiments, we have chosen a common temperature $\beta = 0.1$ to balance the trade-offs between the main and auxiliary objectives. W-COS-DPO achieves smooth and comprehensive Pareto fronts (\emph{cf.} \cref{fig:llm}) with higher hypervolume metrics and less training time (\emph{cf.} Table~\ref{tab:llm}) for both LLM architectures compared to the baselines, demonstrating the effectiveness of our method in the large-scale LLM alignment tasks. 
Notably, as W-COS-DPO tackles a intrinsically more challenging ``meta-learning'' problem and thus demands more expressive power, our method is less prone to overfitting and more robust to the choice of the hyperparameters compared to the baselines.
Several studies on the hyperparameters are provided in \cref{app:ablation}.

\section{Discussion}

In this work, we propose the COS-DPO framework for multi-objective fine-tuning, which is inspired by the DPO framework to profile the Pareto front for a wide range of multi-objective fine-tuning problems. Our method enjoys an efficient one-shot training process by conditioning the model on the importance weights $\vw$ across auxiliary objectives (Weight-COS-DPO), and further also on the temperature $\vbeta$ (Temperature-COS-DPO) to achieve the desired trade-offs between the main and auxiliary objectives. 

We demonstrated the effectiveness and efficiency of Weight-COS-DPO in handling high-dimensional MOFT problems in both the LTR fine-tuning and large-scale LLM alignment tasks, displayed the post-training control by the linear transformation property, and empirically validated the feasibility of Temperature-COS-DPO. Our newly proposed Hyper Prompt Tuning technique also provides a novel way to incorporate continuous information into the LLM. 

We expect to incorporate other possible MOO and LTR techniques in the Weight-COS-DPO method and further explore the potential of the Temperature-COS-DPO method in various, more complicated multi-objective fine-tuning problems in future works.

\begin{acknowledgements} 
       This work was conducted during Yinuo Ren's internship at Amazon. We thank Subhajit Sanyal, Yeshwant Dattatreya, and the whole Amazon Search team for their support and constructive feedback. 
\end{acknowledgements}

\bibliography{reference}

\newpage

\onecolumn

\title{COS-DPO: Conditioned One-Shot Multi-Objective Fine-Tuning Framework\\(Supplementary Material)}
\maketitle

\appendix

\section{Additional Experiment Details}
\label{app:exp}

In this section, we present additional details and results of the experiments conducted in the main text, including further descriptions of the baseline implementations, the penalization terms in the pratical implementation of Weight-COS-DPO and Temperature-COS-DPO. We will also provide a brief discussion on the hypervolume metric as the evaluation metric for the multi-objective fine-tuning (MOFT) methods, and detailed descriptions of the datasets used in the experiments. Additional experimental results of post-training control over trade-offs and the Temperature-COS-DPO method are also provided.

\subsection{Baseline Implementations}
\label{app:baselines}

In the following, we will introduce and discuss the baseline methods used in the experiments in detail.

\paragraph{DPO Linear Scalarization (DPO-LS).}

Given the base model $s_0$, for each weight vector $\vw \in \R^m$, the DPO-LS method trains the new model $s_\theta$ with the loss function $\Ls_{\vw}$ and obtain $s_{\theta, \vw}$ defined as 
\begin{equation*}
    \begin{aligned}
        s_{\theta, \vw}  
        =& \argmin_{s_{\theta}} \Ls_{\vw}(s_\theta; s_0, \vbeta, \gD_{\mathrm{MOFT}}) \\
        =& \argmin_{s_{\theta}} \vw^\top \bm\Ls_{\rm LiPO}(s_\theta ; s_0, \vbeta, \gD_{\mathrm{MOFT}}).
    \end{aligned}
\end{equation*}
This model is a na\"ive generalization from the linear scalarization method in the MOO literature to the MOFT problem, and the main drawback is that it needs as many training jobs and models as the number of sampled weight vectors, which is computationally expensive.

\paragraph{DPO Soup.}

The DPO Soup~\citep{rame2024rewarded} model first trains $m$ models $s_{\theta, \ve_i}$ for each unit vector $\ve_i$ in the $m$-dimensional space, \emph{i.e.}, $m$ DPO models w.r.t. the $m$ auxiliary objectives, respectively, and then linearly combines the $m$ models to obtain the final model with the weight vector $\vw$ in the parameter space.

The DPO Soup method offers a more efficient way to combine the models trained with different auxiliary objectives, but it still requires $m$ training jobs and models for each auxiliary objective, and the performance of this model is largely dependent on the landscape of the parameter space of the neural network architecture.

As depicted in \cref{fig:ltr_2}, the Pareto front obtained by the DPO Soup method may present unexpected curves, and \cref{fig:llm} shows that the DPO Soup method may even exhibit mode collapse for certain combinations.

\paragraph{MO-DPO.}

The MO-DPO method also starts with the training of $m$ models $s_{\theta, \ve_i}$ for each unit vector $\ve_i$ in the $m$-dimensional space, and then instead of linearly combining the parameters, MO-DPO conducts a new training job for each weight vector $\vw \in \R^m$ with the following loss function:
\begin{equation*}
    -\Ls_{\mathrm{MO{\text -}DPO}}(s_\theta ; s_0, \vbeta, \gD_{\mathrm{MOFT}}) 
    = \E \left[
    \sum_{i=1}^{n} \overline z_{i}^{j} \log \dfrac{\exp\left( \beta_j r_{\theta, \vw}^{\rm MO{\text-}DPO} \right)}{\sum_{i'=1}^{n} \exp\left(\beta_j r_{\theta, \vw}^{\rm MO{\text-}DPO}\right)}\right],
\end{equation*}
where, for an arbitrary $i \in [m]$, $r_{\theta, \vw}^{\rm MO{\text-}DPO}$ is defined as
\begin{equation}
    r_{\theta, \vw}^{\rm MO{\text-}DPO} := \dfrac{1}{w_i} \left(s_\theta(\vy | \vx) - s_0(\vy | \vx) - \sum_{i' \neq i} w_{i'} \big(s_{\theta, \ve_i'}(\vy | \vx) - s_0(\vy | \vx)\big)\right).
    \label{eq:mo_dpo}
\end{equation}

\begin{wrapfigure}{r}{0.5\textwidth}
    \centering
    \includegraphics[width=.47\textwidth]{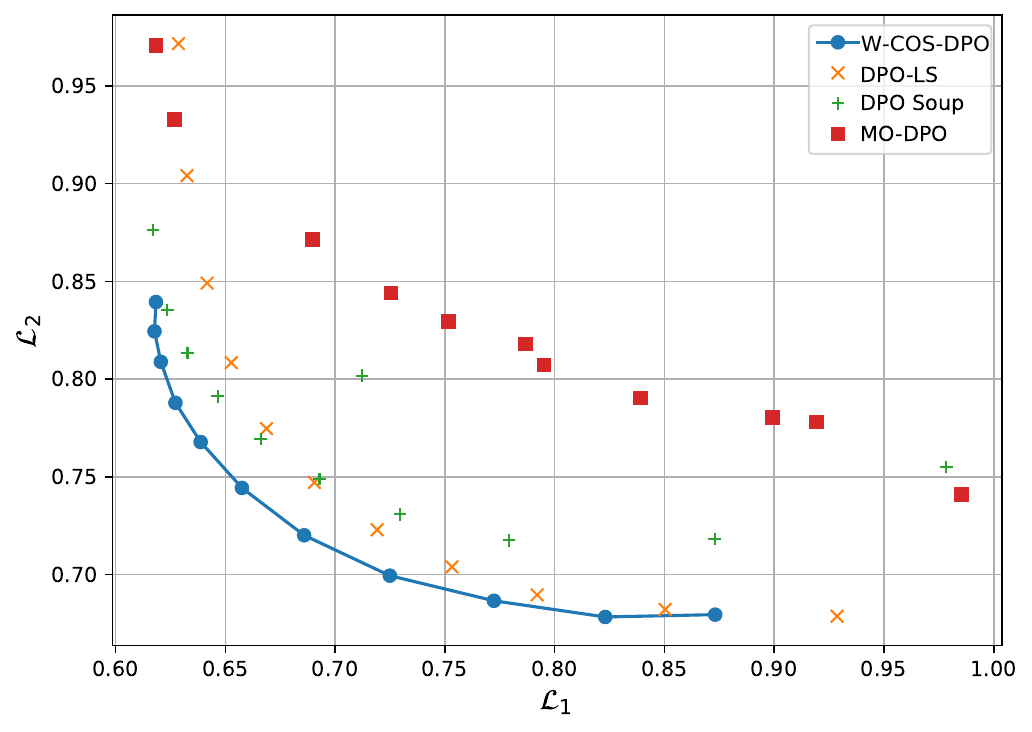}
    \caption{Comparison of Pareto fronts obtained by Weight-COS-DPO and the baselines on the PKU-SafeRLHF dataset with the GPT-2 model, including the MO-DPO method. }
    \label{fig:llm_modpo}
\end{wrapfigure}

As MO-DPO requires $m$ training jobs and one additional training job for each weight vector, it may require more training time and computational resources compared to the DPO-LS and DPO Soup methods.

For the LLM alignment task, we observe MO-DPO suffers from unstable training caused by the $1/w_i$ vector in the expression~\eqref{eq:mo_dpo} especially when $w_i$ is close to zero and exhibits less competitive performance. The results are shown in \cref{fig:llm_modpo}.
We suspect that the conflict between the prompt tuning and the MO-DPO method may lead to the suboptimal performance of MO-DPO in the LLM alignment task and thus do not present the results in the main text (\emph{cf.} \cref{fig:llm}).

The COS-DPO framework is designed to address the limitations of the existing methods and provide a more efficient and effective way to profile the Pareto front of the MOFT problems.

\subsection{Penalization}
\label{app:penalization}

In practice, in order to foster the exploration of the Pareto front, one may also incorporate artificial penalization terms to the loss function, such as the cosine similarity between the loss vector $\bm\Ls(s_\theta; s_0, \vbeta, \gD_{\mathrm{MOFT}})$ of the model and the weight vector~\citep{ruchte2021scalable}:
\begin{equation}
    \begin{aligned}
        &-\gG_{\vw}(s_{\theta}; s_0, \vbeta) \\
        =& \cos \langle\vw, - \bm\Ls_{\rm LiPO}(s_\theta(\cdot, \vw | \vx); s_0, \vbeta, \gD_{\mathrm{MOFT}})\rangle\\
        =& - \dfrac{\vw^\top \bm\Ls_{\rm LiPO}(s_\theta(\cdot, \vw | \vx); s_0, \vbeta, \gD_{\mathrm{MOFT}})}{\|\vw\| \|\bm\Ls_{\rm LiPO}(s_\theta(\cdot, \vw | \vx); s_0, \vbeta, \gD_{\mathrm{MOFT}})\|},
    \end{aligned}
    \label{eq:penalization}
\end{equation}
where $\langle \cdot, \cdot \rangle$ denotes the angle between two vectors.

This penalization term intuitively confines the loss vector $\bm\Ls_{\rm LiPO}$ to converging along the direction of the weight vector $\vw$, empowering possible profiling of concave Pareto fronts~\citep{lin2019pareto}.

We present the practical implementation of the penalization terms in the Weight-COS-DPO and Temperature-COS-DPO methods in the following.

\begin{itemize}
    \item For Weight-COS-DPO, the penalized loss function, modified from~\eqref{eq:W-COS_loss}, is thus defined as
    \begin{equation}
        \begin{aligned}
            &\Ls_{\mathrm{I\hyphen COS}}(s_\theta; s_0, \vbeta, \gD_{\mathrm{MOFT}}, \valpha, \lambda)\\ 
            := &\E_{\vw \sim \mathrm{Dir}(\valpha)}[\Ls_{\vw}(s_\theta(\cdot, \vw | \vx); s_0, \vbeta, \gD_{\mathrm{MOFT}})+ \lambda  \gG_{\vw}(s_\theta(\cdot, \vw | \vx); s_0, \vbeta) ],
        \end{aligned}
        \label{eq:W-COS_loss_penalization}
    \end{equation}
    and the corresponding algorithm is presented in \cref{alg:COS-DPO-pen}. 
    \item For Temperature-COS-DPO, the penalized loss function, modified from~\eqref{eq:T-COS_loss}, is defined as
    \begin{equation*}
        \begin{aligned}
            &\Ls_{\mathrm{T\hyphen COS}}(s_\theta; s_0, \gD_{\mathrm{MOFT}}, \valpha, \lambda)\\ 
            := &\E_{\vbeta\sim\gD(\vbeta), \vw \sim \mathrm{Dir}(\valpha)}[\Ls_{\vw}(s_\theta(\cdot, \vw, \vbeta | \vx); s_0, \gD_{\mathrm{MOFT}}) + \lambda  \gG_{\vw}(s_\theta(\cdot, \vw, \vbeta | \vx); s_0) ],
        \end{aligned}
    \end{equation*}
    and the corresponding algorithm is presented in \cref{alg:temp_COS-DPO-pen}.
\end{itemize}

\IncMargin{1.5em}
\begin{algorithm}[ht]
    \caption{Weight-COS-DPO with Penalization}
    \label{alg:COS-DPO-pen}
    \Indm
    \KwData{Base model $s_0(\vy | \vx)$, dataset $\gD_{\mathrm{MOFT}}$ concentration parameter $\valpha$, temperature $\vbeta$, penalization coefficient $\lambda$(Training); scale $c$, weight vector $\vw$ (Post-Training Control).}
    \KwResult{Fine-Tuned model $s_{\theta, \vbeta}(\cdot, \cdot | \vx)$.}
    \Indp
    \tcp{Training}
    \For{$e = 1$ \KwTo $N_{\rm steps}$}{
        Sample $\vw' \sim \mathrm{Dir}(\valpha)$\;
        $\theta \leftarrow \theta - \eta \nabla_{\theta} [\Ls_{\vw}(s_\theta(\cdot, \vw' | \vx); s_0, \vbeta, \gD_{\mathrm{MOFT}}) + \lambda \gG_{\vw}(s_\theta(\cdot, \vw' | \vx); s_0, \vbeta)];$
    }
    \tcp{Post-Training Control}
    $s_{\theta, c\vbeta}(\vy, \vw | \vx) \leftarrow \left(1 - 1 / c\right) s_{\mathrm{base}}(\vy | \vx) +  s_{\theta}(\vy, \vw | \vx) / c$.
\end{algorithm}
\DecMargin{1.5em}

\IncMargin{1.5em}
\begin{algorithm}[ht]
    \caption{Temperature-COS-DPO with Penalization}
    \label{alg:temp_COS-DPO-pen}
    \Indm
    \KwData{Base model $s_0(\vy | \vx)$, dataset $\gD_{\mathrm{MOFT}}$, concentration parameter $\valpha$, temperature distribution $\gD_\vbeta$, penalization coefficient $\lambda$; temperature $\vbeta$, weight vector $\vw$ (Post-Training Control).}
    \KwResult{Fine-Tuned Model $s_{\theta}(\cdot, \cdot, \cdot | \vx)$.}
    \Indp
    \tcp{Training}
    \For{$e = 1$ \KwTo $N_{\rm steps}$}{
        Sample $\vw' \sim \mathrm{Dir}(\valpha)$, $\vbeta' \sim \gD$\;
        $\theta \leftarrow \theta - \eta \nabla_{\theta} [
                \Ls_{\vw}(s_\theta(\cdot, \vw', \vbeta' | \vx); s_0, \gD_{\mathrm{MOFT}}) + \lambda \gG_{\vw}(s_\theta(\cdot, \vw', \vbeta' | \vx); s_0)];$
    }
    \tcp{Post-Training Control}
    $s_{\theta}(\vy, \vw, \vbeta | \vx) \leftarrow \big(1 - \tfrac{1}{\|\vbeta\|_1}\big) s_0(\vy|\vx) + \tfrac{1}{\|\vbeta\|_1} s_{\theta}\left(\vy, \vw, \tfrac{\vbeta}{\|\vbeta\|_1}|\vx\right)$.
\end{algorithm}
\DecMargin{1.5em}

We have also included the post-training control in both algorithms to scale the output of the model with the temperature $\vbeta$ and the weight vector $\vw$.

\subsection{Hypervolume Metric}
\label{app:hypervolume}

In MOO, the quality of an approximate Pareto front $\hat{\gP}$ is often evaluated using the \emph{hypervolume (HV)} metric, which measures the volume of the region dominated by $\hat{\gP}$ relative to a predefined reference point $\boldsymbol{r}$. This reference point is typically chosen to be a point that is worse than all solutions in the objective space. The hypervolume provides an aggregate measure of performance by capturing how well $\hat{\gP}$ extends toward optimal trade-offs among objectives.

The definition of the hypervolume differs based on whether the objective functions are being maximized or minimized:
\begin{itemize}
    \item For maximization problems, where higher values are preferred, the hypervolume is computed as:

    \begin{equation}
        \mathrm{HV}(\hat \gP, \boldsymbol r) = \int_{\boldsymbol{x} \succeq \boldsymbol{r}} \bm{1}_{\exists \boldsymbol{p} \in \hat{\gP}, \boldsymbol{p} \succeq \boldsymbol{x}} \, \mathrm{d}\boldsymbol{x}.
    \end{equation}
  
    In this case, the hypervolume measures the volume of the region above the reference point $\boldsymbol{r}$ that is dominated by the approximate Pareto front $\hat{\gP}$.
    \item For minimization problems, where lower values are preferred, the hypervolume is defined as:
  
    \begin{equation}
        \mathrm{HV}(\hat \gP, \boldsymbol r) = \int_{\boldsymbol{x} \preceq \boldsymbol{r}} \bm{1}_{\exists \boldsymbol{p} \in \hat{\gP}, \boldsymbol{p} \preceq \boldsymbol{x}} \, \mathrm{d}\boldsymbol{x}.
    \end{equation}
  
    Here, the hypervolume represents the volume of the region below the reference point $\boldsymbol{r}$ that is dominated by $\hat{\gP}$.
\end{itemize}

Intuitively, a larger hypervolume value indicates a better approximation of the true Pareto front $\gP$, as it suggests that $\hat{\gP}$ spans a larger and more favorable region in the objective space, and the optimal Pareto front $\gP$ possesses the maximum hypervolume.

\subsection{Datasets}
\label{app:datasets}

We provide additional details on the datasets used in the experiments.

\paragraph{LTR Fine-Tuning Task.}

In this task, $\vx^{(k)}$ in $\gD_{\rm MOFT}$ denotes a query, and $\vy_i^{(k)}$ denotes the feature vector of the $i$-th document, and $z_i^{j,(k)}$ denotes the score of the $i$-th document w.r.t. the $j$-th aspect. 

The goal of LTR is to provide a ranking $\vpi$ of the documents w.r.t. the scores $z_i^{j,(k)}$ for each query $\vx^{(k)}$, that maximizes the \emph{Normalized Discounted Cumulative Gain (NDCG)}~\citep{wang2013theoretical} metric, defined as
\begin{equation}
    \mathrm{NDCG}^j\mathrm{@k}(\vpi) = \E_{(\vx, \ety, \vz^j)}\left[\frac{\mathrm{DCG@k}(\vpi,\vz^{j})}{\max_{\vpi'} \mathrm{DCG@k}(\vpi',\vz^{j})}\right],
    \label{eq:ndcg}
\end{equation}
where the $\mathrm{DCG@k}$, the discounted cumulative gain for the first $k$ items, is defined as
$$
    \mathrm{DCG@k}(\vpi,\vz^{j}) = \sum_{i=1}^{k} \frac{z_{\pi_i}^{j}}{\log_2(i+1)}.
$$
This metric intuitively measures the quality of the ranking $\vpi$ of the documents w.r.t. the scores $z_i^{j,(k)}$ for each query $\vx^{(k)}$ by assigning higher weights to the top-ranked documents, normalized by the ideal ranking.

We adopt the Microsoft Learning-to-Rank Web Search (MSLR-WEB10K) dataset~\citep{qin2013introducing} for the LTR task.
The MSLR-WEB10K dataset consists of 10,000 groups ($N = 10^4$), each containing a list of webpages retrieved by the search engine in response to the query $\vx^{(k)}$ and the corresponding features extracted from the webpage. Following the practice by~\citet{mahapatra2023querywise}, we treat the first 131 features as the feature vector ($\vy_i^{(k)} \in \R^{131}$). We also identify the relevance label $\in [0:4]$ as the main objective used to train the base model, and the last 5 features, \emph{viz.} (I) Query-URL Click Count, (II) URL Click Count, (III) URL Dwell Time, (IV) Quality Score 1, (V) Quality Score 2, with the relevance label, as 5 different auxiliary objectives ($m=5$) for fine-tuning.  The dataset is split into training (60\%), validation (20\%), and test (20\%) datasets, and all results shown are on the test split. 

We first train w.r.t. the relevance label sufficiently and treat it as our base model $s_0(\vy)$.

\paragraph{LLM Alignment Task.} 

In this task, $\vx^{(k)}$ in $\gD_{\rm MOFT}$ denotes a prompt, and $\vy_i^{(k)}$ denotes the response generated by the LLM, and $z_i^{j,(k)}$ denotes the score of the $i$-th response w.r.t. the $j$-th aspect. The goal is to align the LLM to generate responses that satisfy the auxiliary objectives (\emph{e.g.}, verboseness, harmlessness, \emph{etc.}) while maintaining its performance on general tasks (\emph{e.g.}, fluency, relevance, \emph{etc.}).

We adopt the PKU-SafeRLHF dataset~\citep{ji2024pku} for experiments, which consists of 83,400 entries, each containing a prompt and a pair of responses ($n = 2$) annotated with preferences w.r.t. both harmlessness and helpfulness ($m = 2$). When the $k$-th response is annotated as more helpful, we assign $z_1^{(k)} = 1$; otherwise, $z_1^{(k)} = 0$. Similarly, when the $k$-th response is annotated as more harmless, we assign $z_2^{(k)} = 1$; otherwise, $z_2^{(k)} = 0$.
The goal is to fine-tune the model to generate responses that are both harmless and helpful as a multi-objective optimization problem. 

\subsection{Remark on Training Time}
\label{app:time}

The training time in Table~\ref{tab:ltr_5} refers to the duration of all training jobs required for computing the 11-point Pareto front. As described in \cref{alg:COS-DPO} or its penalized version \cref{alg:COS-DPO-pen}, in each epoch during the Weight-COS-DPO training, we first sample a single weight vector $\vw$ and then compute the loss $\Ls_{\mathrm{W\hyphen COS}}$ and back-propagate the gradients. Therefore, the training does not introduce additional computational cost compared to the training w.r.t. a single objective. 

However, Weight-COS-DPO may require more training epochs to converge due to the exploration of the Pareto front. In practice, we find that the Weight-COS-DPO framework converges rapidly, and the training time may only be slightly longer than that of a single model training.

\subsection{Hyperparameter Studies}
\label{app:ablation}

In this section, we provide the studies of the hyperparameters of the COS-DPO framework, including the sensitivity of the concentration parameter $\valpha$ in the Dirichlet distribution, the depth of the model, and the performance of two different NN parametrizations $s_{\theta, \vw, \vbeta}(\cdot, \cdot | \vx)$, namely training-from-scratch and augmentation network.

\begin{figure*}[!p]
    \begin{subfigure}{\textwidth}
        \centering
        \begin{minipage}{.57\textwidth}
            \centering
            \includegraphics[width=.8\linewidth]{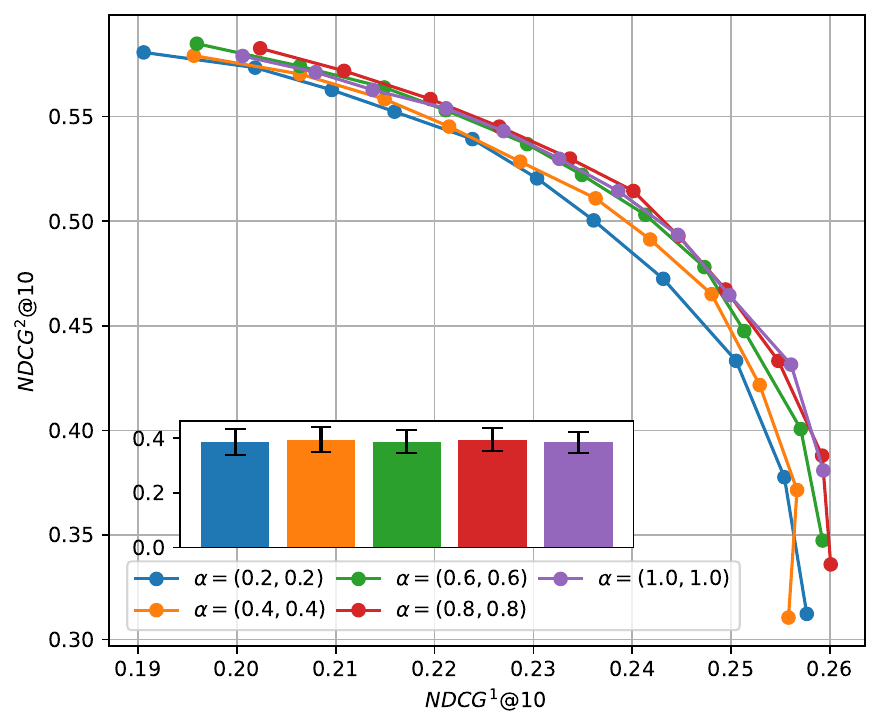}
        \end{minipage}
        \begin{minipage}{.37\textwidth}
            \centering
            \begin{tabular}{c|c}
                \toprule
                \bf $\valpha$ & \bf Hypervolume \\
                \midrule
                $(0.2, 0.2)$ & $1.446\times 10^{-1}$\\[.5ex]
                $(0.4, 0.4)$ & $1.445 \times 10^{-1}$\\[.5ex]
                $(0.6, 0.6)$ & $1.471 \times 10^{-1}$\\[.5ex]
                $(0.8, 0.8)$ & $\bf 1.473 \times 10^{-1}$\\[.5ex]
                $(1.0, 1.0)$ & $1.463 \times 10^{-1}$\\
                \bottomrule
            \end{tabular}
        \end{minipage}
        \caption{$\valpha = (\alpha, \alpha)$ for $\alpha \in \{0.2, 0.4, 0.6, 0.8, 1.0\}$.}
        \label{fig:alpha_1}
    \end{subfigure}

    \begin{subfigure}{\textwidth}
        \centering
        \begin{minipage}{.57\textwidth}
            \centering
            \includegraphics[width=.8\linewidth]{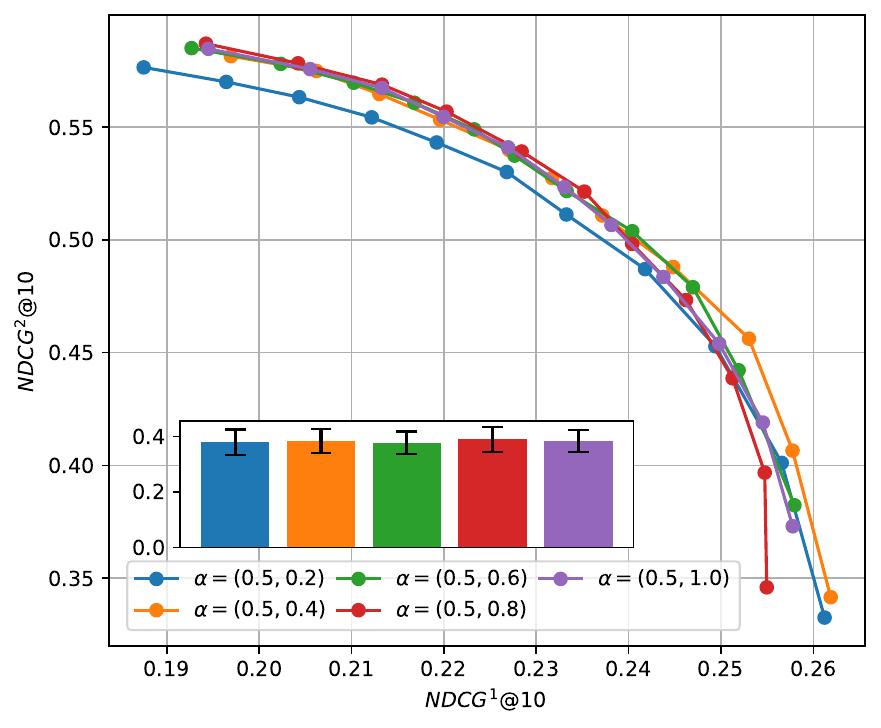}
        \end{minipage}
        \begin{minipage}{.37\textwidth}
            \centering
            \begin{tabular}{c|c}
                \toprule
                \bf $\valpha$ & \bf Hypervolume \\
                \midrule
                $(0.5, 0.2)$ & $1.451 \times 10^{-1}$\\[.5ex]
                $(0.5, 0.4)$ & $\bf 1.474 \times 10^{-1}$\\[.5ex]
                $(0.5, 0.6)$ & $1.466 \times 10^{-1}$\\[.5ex]
                $(0.5, 0.8)$ & $1.458 \times 10^{-1}$\\[.5ex]
                $(0.5, 1.0)$ & $1.464 \times 10^{-1}$\\
                \bottomrule
            \end{tabular}
        \end{minipage}
        \caption{$\valpha = (0.5, \alpha)$ for $\alpha \in \{0.2, 0.4, 0.6, 0.8, 1.0\}$. }
        \label{fig:alpha_2}
    \end{subfigure}

    \begin{subfigure}{\textwidth}
        \centering
        \begin{minipage}{.57\textwidth}
            \centering
            \includegraphics[width=.8\linewidth]{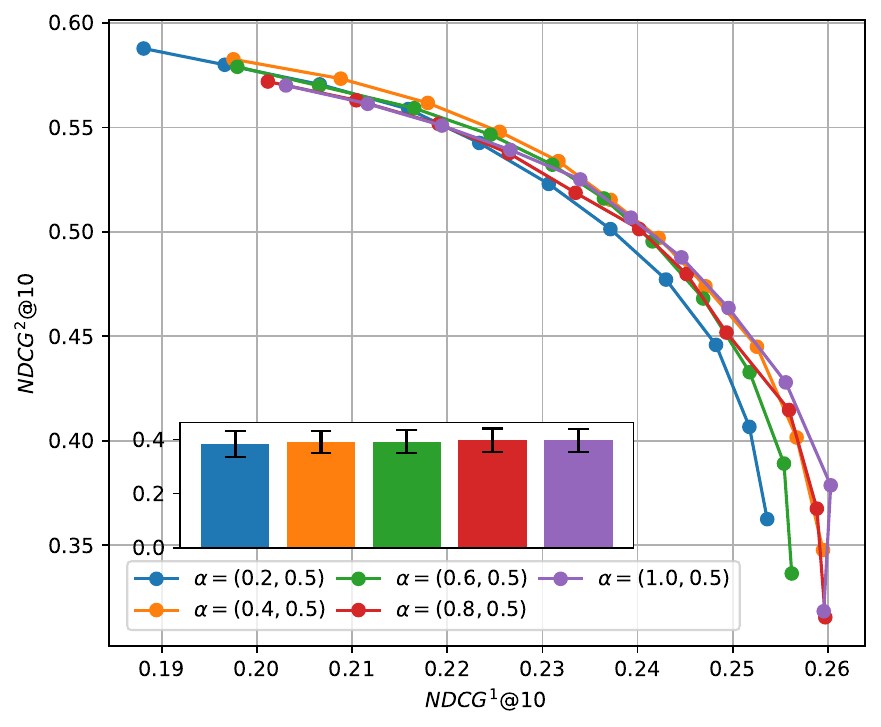}
        \end{minipage}
        \begin{minipage}{.37\textwidth}
            \centering
            \begin{tabular}{c|c}
                \toprule
                \bf $\valpha$ & \bf Hypervolume \\
                \midrule
                $(0.2, 0.5)$ & $1.447 \times 10^{-1}$\\[.5ex]
                $(0.4, 0.5)$ & $\bf 1.468 \times 10^{-1}$\\[.5ex]
                $(0.6, 0.5)$ & $1.445 \times 10^{-1}$\\[.5ex]
                $(0.8, 0.5)$ & $1.444 \times 10^{-1}$\\[.5ex]
                $(1.0, 0.5)$ & $1.445 \times 10^{-1}$\\
                \bottomrule
            \end{tabular}
        \end{minipage}
        \caption{$\valpha = (\alpha, 0.5)$ for $\alpha \in \{0.2, 0.4, 0.6, 0.8, 1.0\}$. }
        \label{fig:alpha_3}
    \end{subfigure}

    \caption{Hyperparameter study on the impact of concentration parameter $\valpha$ on the Pareto fronts obtained by Weight-COS-DPO on the MSLR-WEB10K dataset (Objective I vs Objective II) with different settings of $\valpha$. The hypervolume metric is shown in the table beside each figure.}
    \label{fig:alpha}
\end{figure*}

\subsubsection{Concentration Parameter \texorpdfstring{$\valpha$}{alpha}}
\label{app:alpha}

The concentration parameter $\valpha$ controls the span of the Dirichlet distribution from which the weight vector $\vw$ is sampled and is the key parameter affecting the performance of the COS-DPO framework that should be carefully selected and validated. By the basic properties of the Dirichlet distribution, suppose $\vw \sim \mathrm{Dir}(\valpha)$, then we have 
\begin{equation*}
    \E[\vw] = \dfrac{\valpha}{\|\valpha\|_1}:= \overline{\valpha}, \quad \var(\vw) = \dfrac{\diag(\overline{\valpha}) - \overline{\valpha} \overline{\valpha}^\top}{\|\valpha\|_1 + 1}.
\end{equation*}
In general, the smaller the $\valpha$, the more likely the weight vector $\vw$ is close to the boundary of the simplex, and the larger the $\valpha$, the more likely the weight vector $\vw$ is concentrated around the expectation $\overline{\valpha}$.

As the COS-DPO framework is generally robust to the choice of the concentration parameter $\valpha$, we conduct ablation studies to investigate the impact of the concentration parameter $\valpha$ on the performance of the COS-DPO framework in different settings.
We first conduct experiments on the MSLR-WEB10K dataset with 2 auxiliary objectives (Query-URL Click Count vs URL Click Count) to investigate the impact of the concentration parameter $\valpha$ on the performance of the COS-DPO framework. The results are shown in \cref{fig:alpha}. The experiment settings and plotting details are the same as in the main text.

As shown in \cref{fig:alpha_1}, as the concentration parameter $\valpha$ decreases, COS-DPO obtains a visually more comprehensive Pareto front thanks to more samples close to the boundary of the simplex. However, it is at the cost of a slightly undertrained model across the simplex, indicated by a lower hypervolume metric. It turns out that the choice of $\valpha$ faces a trade-off between the diversity of the samples and the overall quality of the fine-tuning, given a fixed training budget. Similar trade-offs are observed in \cref{fig:alpha_2} and~\ref{fig:alpha_3} when only one dimension of the concentration parameter $\valpha$ is varied.

\begin{figure*}[!t]
    \centering
    \begin{subfigure}{.47\textwidth}
        \includegraphics[width=\linewidth]{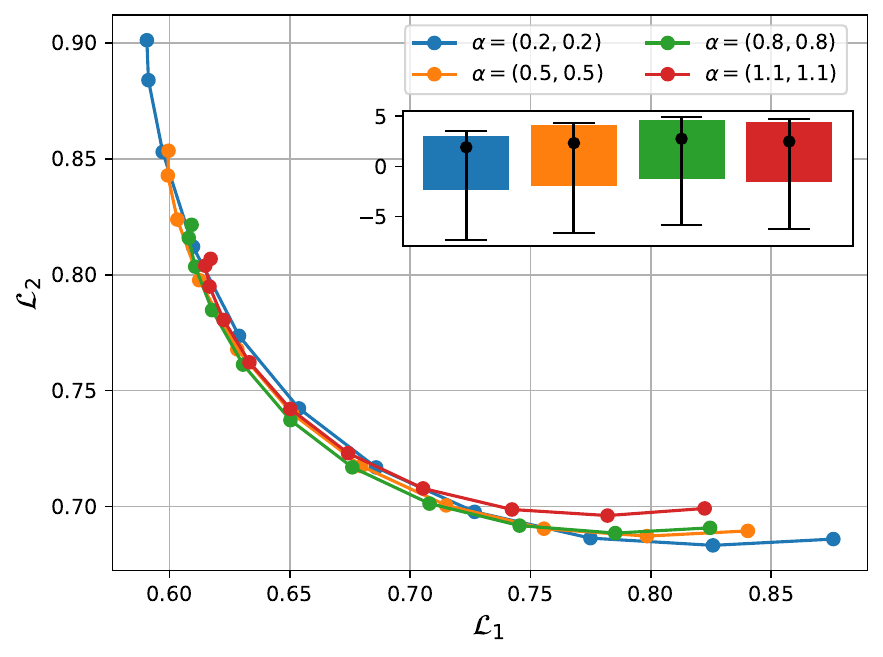}
        \caption{GPT2}
        \label{fig:alpha_gpt2}
    \end{subfigure}
    \hfill
    \begin{subfigure}{.47\textwidth}
        \includegraphics[width=\linewidth]{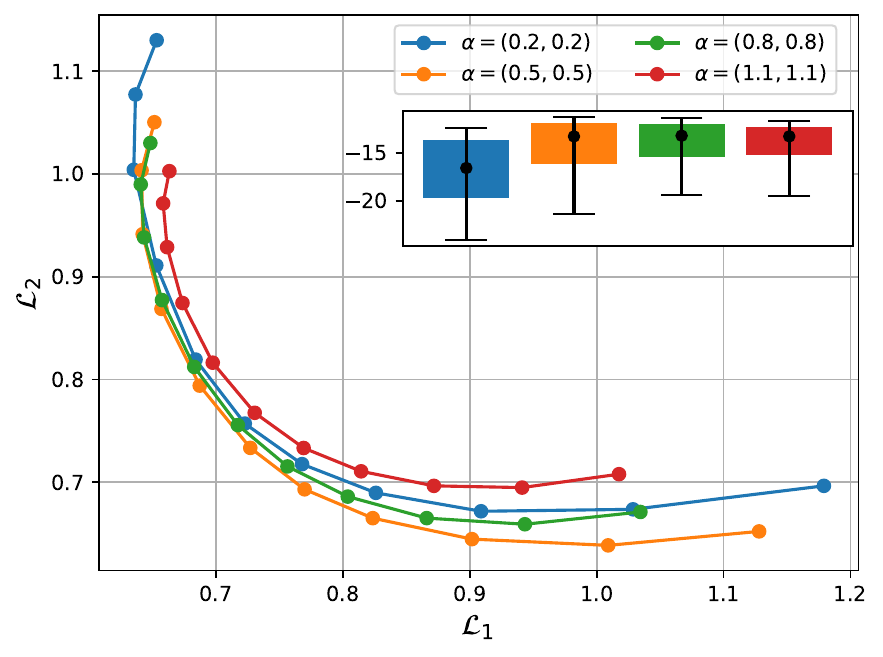}
        \caption{Alpaca-7B-Reproduced}
        \label{fig:alpha_llama2}
    \end{subfigure}
    \caption{Hyperparameter study on the impact of the concentration parameter $\valpha$ on the Pareto fronts obtained by Weight-COS-DPO on the PKU-SafeRLHF dataset.}
    \label{fig:alpha_llm}
\end{figure*}

We also conducted experiments on the PKU-SafeRLHF dataset to investigate the impact of the concentration parameter $\valpha$ on the performance of the COS-DPO framework on the LLM alignment task. The results are shown in \cref{fig:alpha_llm}. A similar pattern is observed in this large-scale task, where a smaller choice of the concentration parameter $\valpha$ leads to a more comprehensive Pareto front. However, it does not necessarily lead to a worse hypervolume metric, suggesting that the performance of COS-DPO here is less hindered by the expressive power of the model, which has already been abundant in the LLM, 
but rather by the diversity of the samples.

\subsubsection{Model Depth}
\label{app:depth}

The depth of the neural network architecture is also crucial for the performance of the COS-DPO framework, as it determines the complexity and the expressiveness of the model.
We also use the MSLR-WEB10K dataset with 2 auxiliary objectives (Query-URL Click Count vs URL Click Count) to investigate the impact of \emph{the model depth} on the performance of the COS-DPO framework. The results are shown in \cref{fig:block_a}, where the depth, referring to the number of transformer layers in the model, is varied from 1 to 5. As shown in the figure, the performance of the COS-DPO framework is first significantly improved and gradually saturated with the increase of depth. Besides, while the hypervolume metric improves, the coverage of the Pareto front does not change significantly with the increase in depth. This suggests that the concentration parameter $\valpha$ may have a more significant impact on the diversity of the samples than the model depth.

\subsubsection{Model Parametrization}
\label{app:ppodpo}

In general, one could adopt one of the two different parametrizations of $s_{\theta}(\cdot, \vw | \vx)$ (or $s_{\theta}(\cdot, \vw, \vbeta | \vx)$, respectively) in the COS-DPO framework. For simplicity, we will only discuss the case of Weight-COS-DPO in the following, and the discussion can be easily extended to the Temperature-COS-DPO method naturally.
\begin{itemize}[leftmargin=*]
    \item{\it \underline{Training-from-Scratch}}:
    The model $s_{\theta}(\cdot, \vw| \vx)$ is a completely separate neural network from the base model $s_0(\vy | \vx)$. Depending on the specific design of the additional inputs $\vw$, the new model may or may not share the same architecture as the base model. The main advantage of this design is that it requires less memory and computation resources~\citep{rafailov2024direct}, and thus is more suitable for large-scale applications, \emph{e.g.}, LLMs. 
    \item{\it \underline{Augmentation Network}}:
    As several works~\citep{chen2024preference,xu2024dpo} argue that DPO is prone to overfitting, one may curb the complexity of the model for the score function $s_\theta(\cdot, \vw | \vx)$ by only adding a first-order correction term to the base model $s_0(\vy | \vx)$ as:
    \begin{equation*}
        s_{\theta}(\vy, \vw | \vx) = s_0(\vy | \vx) + \Delta s_{\theta}(\vy, \vw | \vx),
    \end{equation*}
    where the parameters in the base model are fixed, and the importance-conditioned design is only applied to the correction term $\Delta s_\theta(\cdot, \vw | \vx)$. This design allows limited modification and reversibility to the base model and is thus suitable for applications where the fine-tuning is limited in budget, frequent, or expected to be minor.
\end{itemize}

\begin{figure}[ht]
    \centering
    \begin{subfigure}{.47\textwidth}
        \centering
        \begin{tikzpicture}
            \node (y) at (-1.2, 2.75) {$\vy|\vx$};

            \node (w) at (-1.2, .75) {$\vw$};

            \node (beta) at (-1.2, -1.25) {$\vbeta$};

            \node[rectangle, rounded corners=.5cm, draw, fit={(0,2) (1.,3.5)}, fill=gray!70!white, fill opacity=0.3, text opacity=1] (s_base) {$s_0$};
            \node[rectangle, rounded corners=.5cm, draw, fit={(0,0) (1.,1.5)}, fill=yellow!70!white, fill opacity=0.3, text opacity=1] (s_theta) {$s_\theta$};

            \node (L) at (2.5, 1.75) {$\Ls_{\mathrm{ICOS}}$};

            \draw[->] (y) -- (s_base);
            \draw[->] (y) to[in=160, out=-20] (s_theta); 
            \draw[->] (w) -- (s_theta);
            \draw[->,dashed] (beta) to[in=200, out=-20] (s_theta);

            \draw[dashed,->] (s_base) to[out=0,in=135] (L);
            \draw[->] (s_theta) to[out=0,in=-135] (L);
        \end{tikzpicture}
        \caption{Training-from-Scratch}
        \label{fig:ICOS_a}
    \end{subfigure}
    \hfill
    \begin{subfigure}{.52\textwidth}
        \centering
        \begin{tikzpicture}
            \node (y) at (-1.2, 2.75) {$\vy|\vx$};

            \node (w) at (-1.2, .75) {$\vw$};

            \node (beta) at (-1.2, -1.25) {$\vbeta$};

            \node[rectangle, rounded corners=.5cm, draw, fit={(0,2) (1.,3.5)}, fill=gray!70!white, fill opacity=0.3, text opacity=1] (s_base) {$s_0$};
            \node[rectangle, rounded corners=.5cm, draw, fit={(0,.25) (1.,1.25)}, fill=yellow!70!white, fill opacity=0.3, text opacity=1] (delta_s_theta) {$\Delta s_\theta$};
            \node[circle, draw, plus, minimum size=.5cm, inner sep=0pt] (oplus) at (2, .75) {};
            \node[rectangle, rounded corners=.5cm, draw, fit={(3.,0) (4.,1.5)}, fill=yellow!70!white, fill opacity=0.3, text opacity=1] (s_theta) {$s_\theta$};

            \node (L) at (5.5, 1.5) {$\Ls_{\mathrm{ICOS}}$};

            \draw[->] (y) -- (s_base);
            \draw[->] (y) to[in=160, out=-20] (delta_s_theta); 
            \draw[->] (w) -- (delta_s_theta);
            \draw[->,dashed] (beta) to[in=200, out=-20] (delta_s_theta);
            \draw[dashed,->] (s_base) to[out=-20,in=90] (oplus);
            \draw[->] (delta_s_theta) -- (oplus);
            \draw[->] (oplus) -- (s_theta);

            \draw[dashed,->] (s_base) to[out=0,in=135] (L);
            \draw[->] (s_theta) to[out=0,in=-135] (L);
        \end{tikzpicture}
        \caption{Augmentation Network}
        \label{fig:ICOS_b}
    \end{subfigure}
    \caption{Illustration of two different parametrizations of the model $s_{\theta}(\cdot, \vw | \vx)$ in the COS-DPO framework. Dashed lines denote that backpropagation is not applied.}
    \label{fig:ICOS}
\end{figure}

The two parametrizations are illustrated in \cref{fig:ICOS_a} and~\ref{fig:ICOS_b}, respectively. Both parametrizations can be seamlessly applied to the COS-DPO framework and easily switch between each other. In all the experiments presented in the main text, we have adopted the training-from-scratch design for the COS-DPO framework. \cref{fig:block_b} shows the results of the COS-DPO framework with the augmentation training design on the same task as the previous ablation studies. Compared with \cref{fig:block_a}, the augmentation training achieves a roughly better performance than the training-from-scratch design with the same depth, coinciding with the intuition that the augmentation training benefited from the information provided by the base model and instead of learning the entire score function $s_{\theta}(\cdot, \vw | \vx)$ from scratch, it only needs to learn the correction term $\Delta s_{\theta}(\cdot, \vw | \vx)$. When the model depth is increased, the performance of the augmentation training is also improved, sharing the same trend as the training-from-scratch design.

\begin{figure*}[!t]
    \centering
    \begin{subfigure}{.47\textwidth}
        \includegraphics[width=\linewidth]{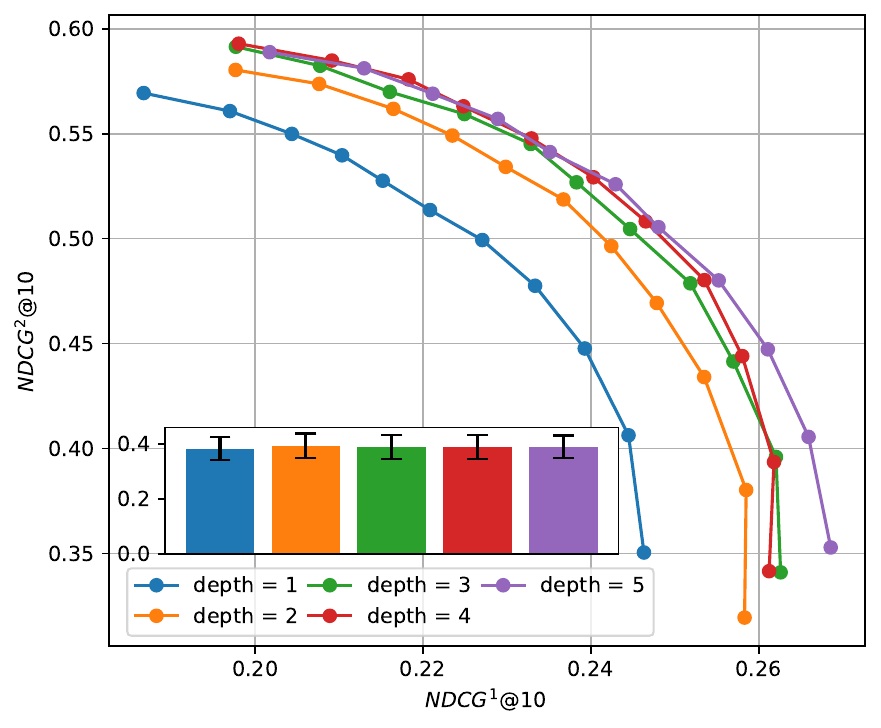}
        \caption{Training-from-Scratch}
        \label{fig:block_a}
    \end{subfigure}
    \hfill
    \begin{subfigure}{.47\textwidth}
        \includegraphics[width=\linewidth]{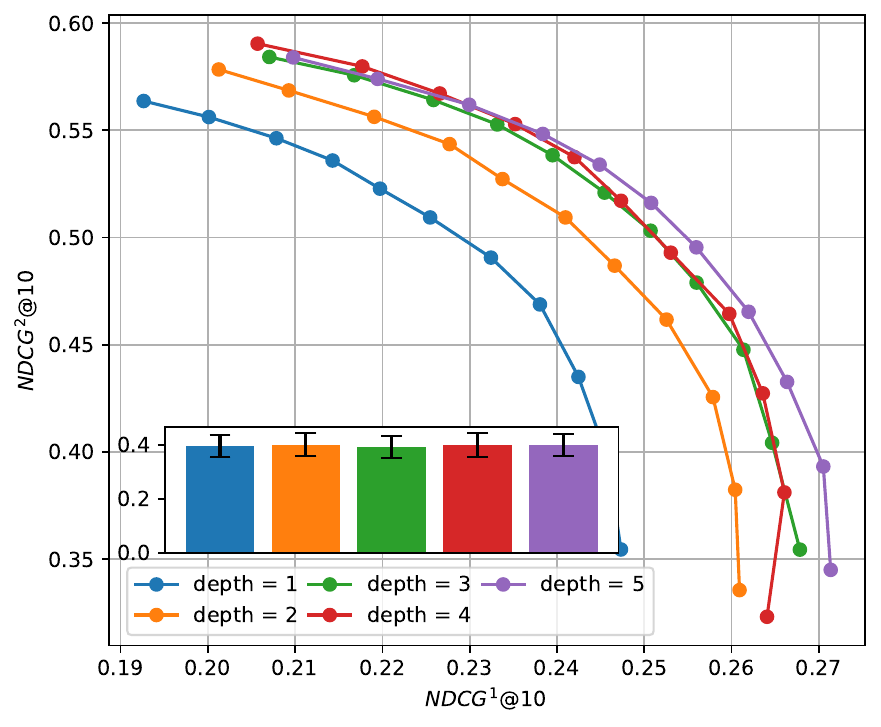}
        \caption{Augmentation Network}
        \label{fig:block_b}
    \end{subfigure}
    \caption{Study on the impact of the model depth and the model parametrizations on the Pareto fronts obtained by Weight-COS-DPO on the MSLR-WEB10K dataset (Objective I vs Objective II).}
    \label{fig:block}
\end{figure*}

\subsection{Additional Results on Linear Transformation Property}
\label{app:post_training}

As discussed in \cref{sec:linear}, the linear transformation property implies that the model can be scaled proportionally by a constant factor $c$ by a simple linear transformation of the output scores. We provide two relevant experiments on the post-training controls of the Weight-COS networks obtained by Weight-COS-DPO based on this property.

\begin{figure}[!t]
    \centering
    \begin{subfigure}{.47\textwidth}
        \centering
        \includegraphics[width=\linewidth]{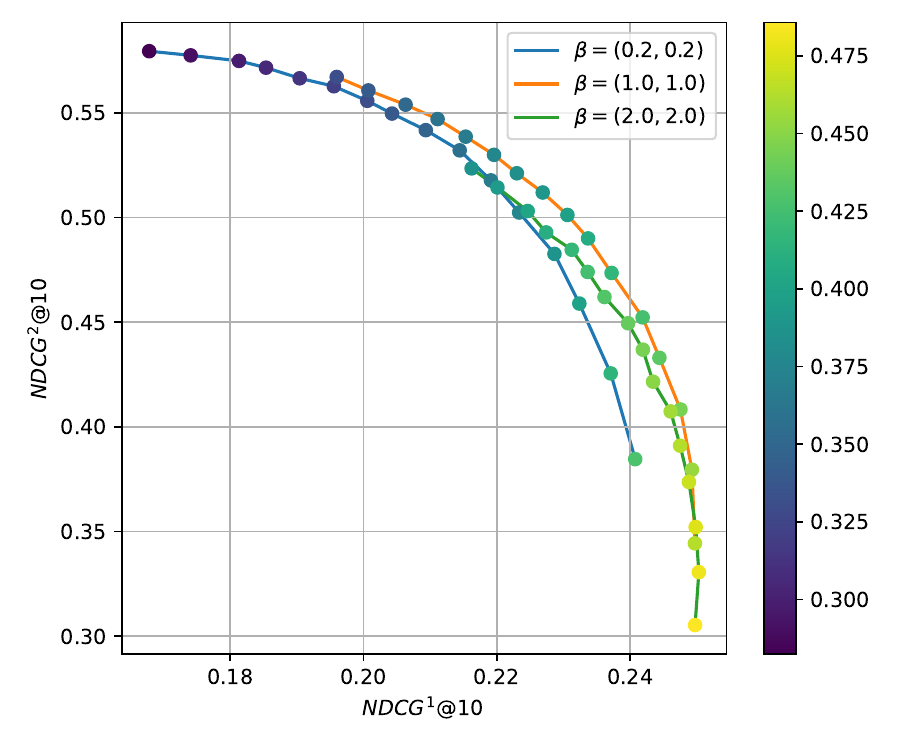}
        \caption{Objective I vs Objective II.}
        \label{fig:beta_1}
    \end{subfigure}
    \hfill
    \begin{subfigure}{.47\textwidth}
        \centering
        \includegraphics[width=\linewidth]{figures/beta_qs_qs2.pdf}
        \caption{Objective IV vs Objective V.}
        \label{fig:beta_2}
    \end{subfigure}
    \caption{Examples of post-training control of Weight-COS networks over temperature $\vbeta$ on the MSLR-WEB10K dataset with 2 auxiliary objectives. Two axes denote the NDCG@10 of the two auxiliary objectives (the higher, the better). The colorbar denotes the NDCG@10 of the main objective.}
    \label{fig:beta}
\end{figure}

\begin{figure}[!t]
    
    \centering
    \begin{subfigure}{.47\textwidth}
        \centering
        \includegraphics[width=\linewidth]{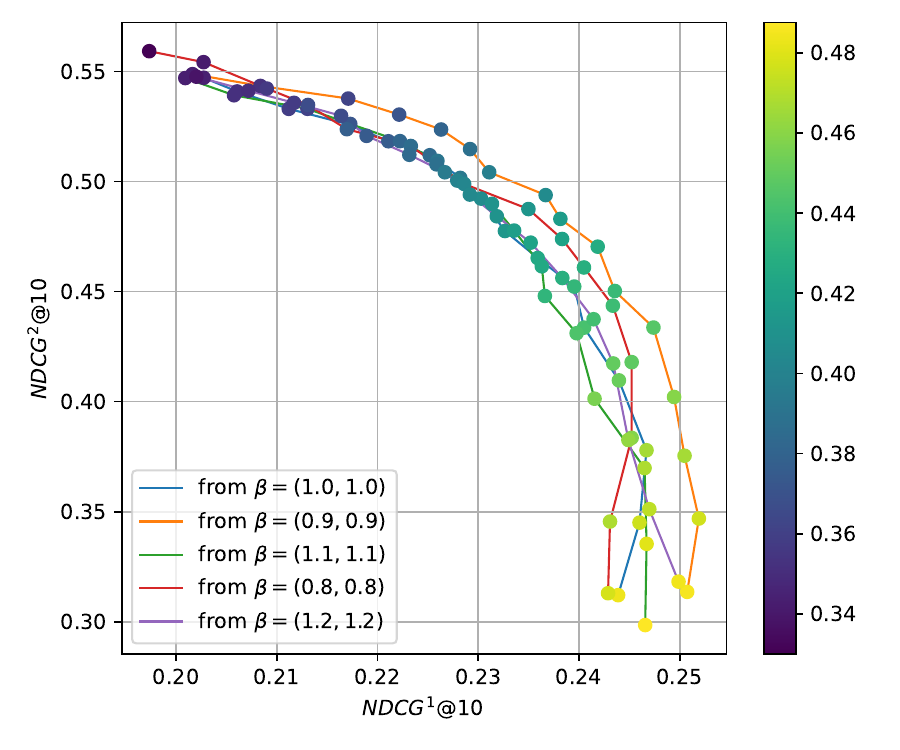}
        \caption{Objective I vs Objective II.}
        \label{fig:from_1}
    \end{subfigure}
    \hfill
    \begin{subfigure}{.47\textwidth}
        \centering
        \includegraphics[width=\linewidth]{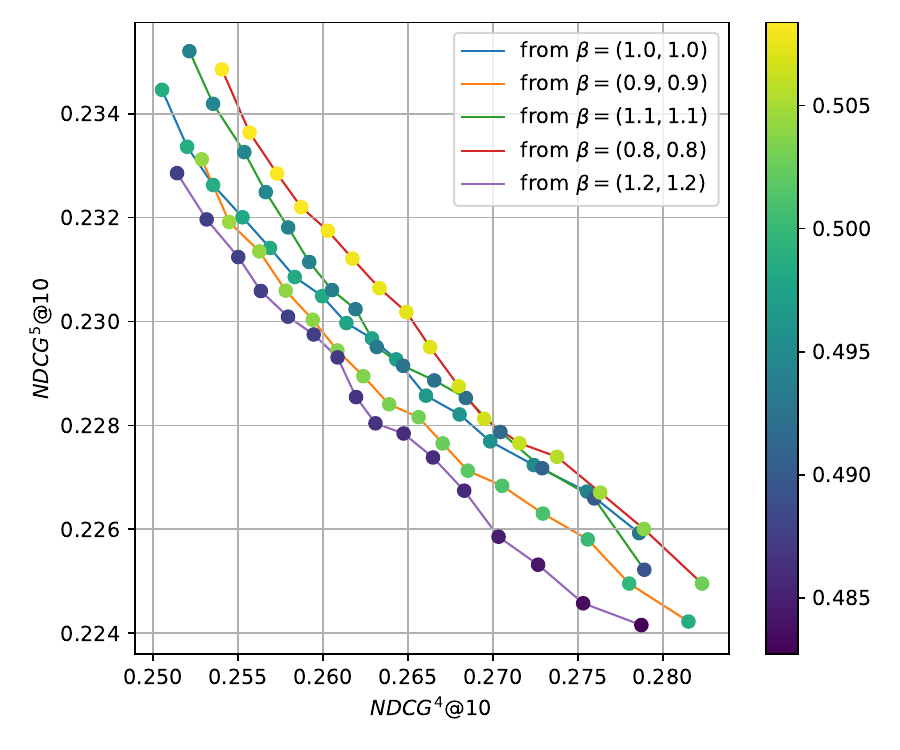}
        \caption{Objective IV vs Objective V.}
        \label{fig:from_2}
    \end{subfigure}
    \caption{Empirical validation of the linear transformation property of Weight-COS networks on the MSLR-WEB10K dataset with 2 auxiliary objectives. The Pareto fronts in the figures are obtained by first training a model with the temperature $\vbeta$ in the legend and then transforming to the same temperature $\vbeta = (1, 1)$ using post-training controls. Two axes denote the NDCG@10 of the two auxiliary objectives (the higher, the better). The colorbar denotes the NDCG@10 of the main objective.}
    \label{fig:from}
\end{figure}

\cref{fig:beta} gives examples of the post-training control over the temperature $\vbeta$ on the MSLR-WEB10K dataset with 2 auxiliary objectives. As the temperature $\vbeta$ increases, the Pareto front shifts towards the direction where the main objective is more emphasized, which is consistent with our expectations. In \cref{fig:beta_2}, the two auxiliary objectives are in balance, and thus, the shifts of the Pareto fronts resemble that depicted in \cref{fig:post_training}. However, in \cref{fig:beta_1}, the unexpected shifting pattern is observed, which may reflect the complex interactions between the main and auxiliary objectives.

\cref{fig:from} provides empirical validation of the linear transformation property on the MSLR-WEB10K dataset with 2 auxiliary objectives. The methodology is that we first train a Weight-COS network with the different temperatures $\vbeta$ ranging from $(0.8, 0.8)$ to $(1.2, 1.2)$, and then transform the Pareto fronts obtained by the trained models to the same temperature $\vbeta = (1, 1)$ using the post-training control as indicated in~\eqref{eq:linear_transform} and \cref{alg:COS-DPO-pen}. The penalization coefficient $\lambda$ is set to $0.05$ in the training.
The results show that the transformed Pareto fronts are roughly aligned with each other, which validates the linear transformation property of the model. The slight deviation may be caused by the noises in the training process and the non-uniqueness of the optimal solutions of the Weight-COS-DPO loss.

\subsection{Additional Results of Temperature-COS-DPO}
\label{app:temp}

All experiments of the Temperature-COS-DPO method are conducted on the MSLR-WEB10K dataset with 2 auxiliary objectives (Quality Score vs Quality Score 2) to investigate the performance of the Temperature-COS-DPO, as it provides better visualization and comparisons of the Pareto fronts with different temperature parameters $\vbeta$.
In particular, we adopt the augmentation network design for Temperature-COS networks for better expressive power and stability. 

\begin{figure*}[t]
    \centering
    \begin{subfigure}{.47\textwidth}
        \includegraphics[width=\linewidth]{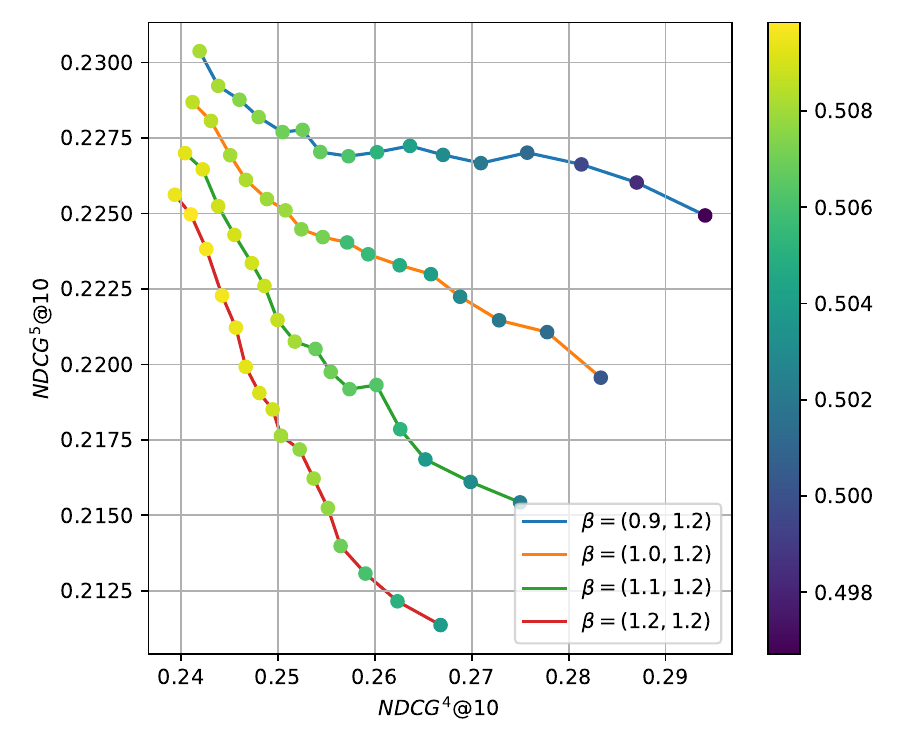}
        \caption{$\vbeta = (\beta, 1.2)$ for $\beta\in\{0.9, 1.0, 1.1, 1.2\}$.}
    \end{subfigure}
    \hfill 
    \begin{subfigure}{.47\textwidth}
        \includegraphics[width=\linewidth]{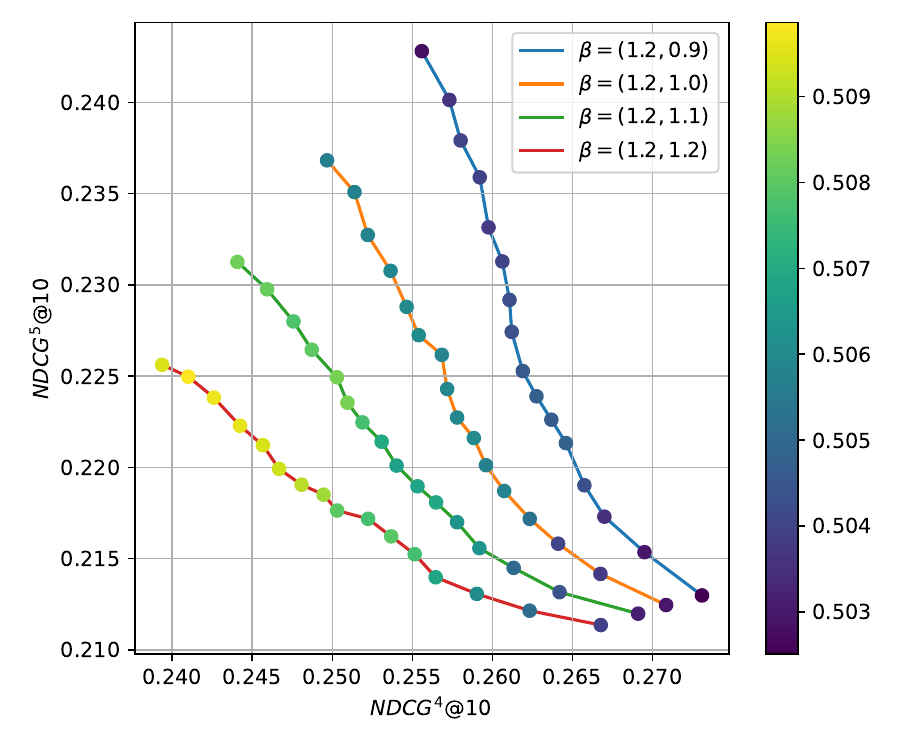}
        \caption{$\vbeta = (1.2, \beta)$ for $\beta\in\{0.9, 1.0, 1.1, 1.2\}$.}
    \end{subfigure}
    \caption{Preliminary results of Temperature-COS-DPO on the MSLR-WEB10K dataset (Objective IV vs Objective V). The colorbar denotes the NDCG@10 of the main objective.}
    \label{fig:temp}
\end{figure*}

We provide the results of the Temperature-COS-DPO method on the LTR fine-tuning task in \cref{fig:temp}. The model depth is chosen to be 5, and the distribution $\gD_\vbeta$ is set to be $\unif([0.67, 1.5]^2)$.
The results demonstrate the Temperature-COS-DPO method is capable of capturing the trade-off between the main objective and the auxiliary objectives for all kinds of temperature configurations $\vbeta$, and the Pareto fronts exhibit expected behaviors with different $\vbeta$. These results suggest that temperature-conditioned one-shot fine-tuning is a promising direction for the COS-DPO framework to achieve more flexible control over the Pareto front.

Given the choices of the temperature parameters, the Pareto fronts in both figures in~\cref{fig:temp} should merge into one single point, which refers to the solution of the single-objective fine-tuning task with certain temperature parameter $\beta$. Although the results are roughly in accordance with the theoretical expectations, there are still small gaps that may be accounted for by the limit of the expressive power of the model and insufficient exploration over the weight vector $\vw$. 

\begin{figure}[p]
    \centering
    \begin{subfigure}{.47\textwidth}
        \includegraphics[width=\linewidth]{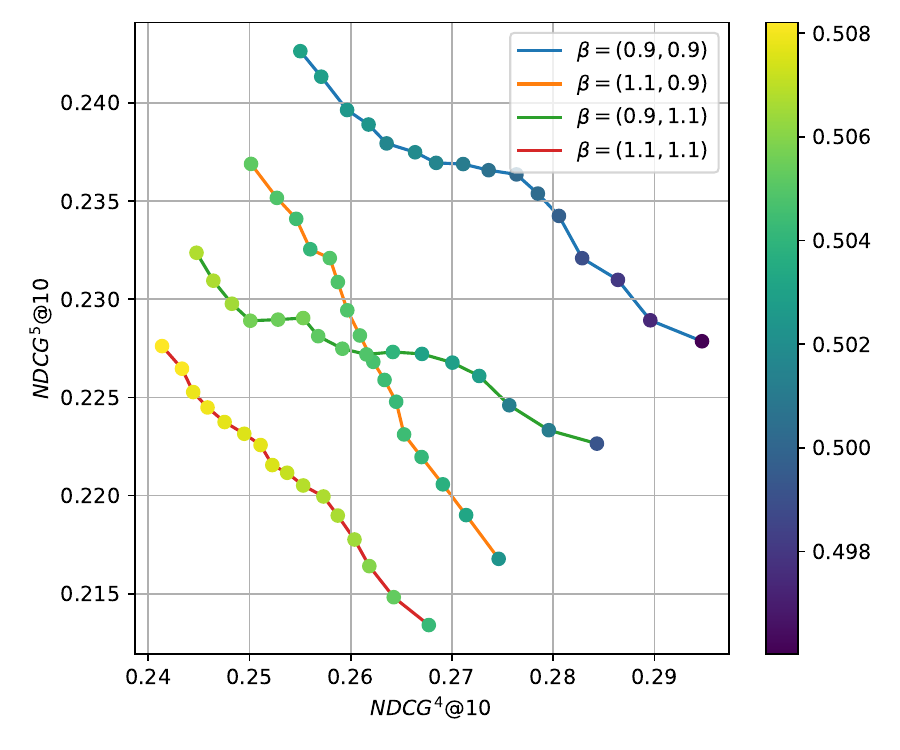}
        \caption{Depth = 2.}
        \label{fig:beta_dep_2}
    \end{subfigure}
    \hfill
    \begin{subfigure}{.47\textwidth}
        \includegraphics[width=\linewidth]{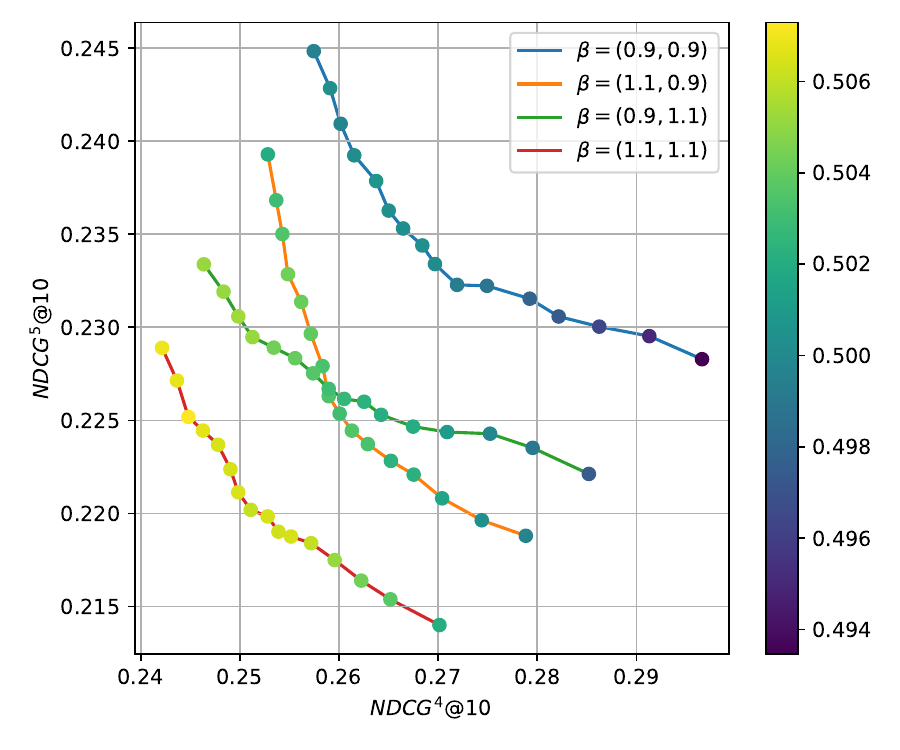}
        \caption{Depth = 3.}
        \label{fig:beta_dep_3}
    \end{subfigure}
    
    \begin{subfigure}{.47\textwidth}
        \includegraphics[width=\linewidth]{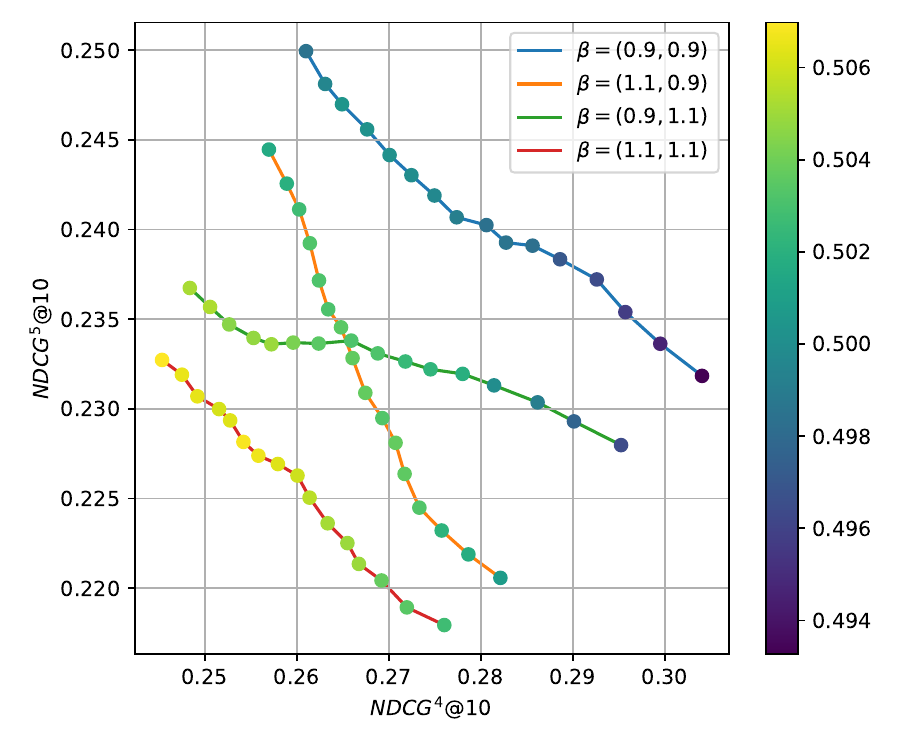}
        \caption{Depth = 4.}
        \label{fig:beta_dep_4}
    \end{subfigure}
    \hfill
    \begin{subfigure}{.47\textwidth}
        \includegraphics[width=\linewidth]{figures/beta_5.pdf}
        \caption{Depth = 5.}
        \label{fig:beta_dep_5}
    \end{subfigure}
    \caption{Ablation study of the impact of model depth on the Pareto fronts obtained by the Temperature-COS-DPO on the MSLR-WEB10K dataset (Objective IV vs Objective V). The colorbar denotes the NDCG@10 of the main objective.}
    \label{fig:temp_depth}
\end{figure}

To explain this, we present ablation studies to investigate the effect of the expressiveness of the model on the performance of the Temperature-COS-DPO. We applied models with 2 to 5 layers of transformer architecture, and the results show that the performance, indicated by the expected behaviors of the Pareto front, is drastically improved with the increase of the number of layers. While swallower models yield Pareto fronts with less expected behaviors and more noise, \emph{e.g.}, the concavity of the Pareto fronts in \cref{fig:beta_dep_3} partially indicates the insufficiency of the training process, the model with 5 layers of transformer architecture in \cref{fig:beta_dep_5} exhibits improved scores and more expected behaviors according to different temperature configurations. This suggests and confirms the intuition that Temperature-COS-DPO requires more expressive structures to capture the complex trade-offs between the main and auxiliary objectives.

\begin{figure}[t]
    \centering
    \begin{subfigure}{.47\textwidth}
        \includegraphics[width=\linewidth]{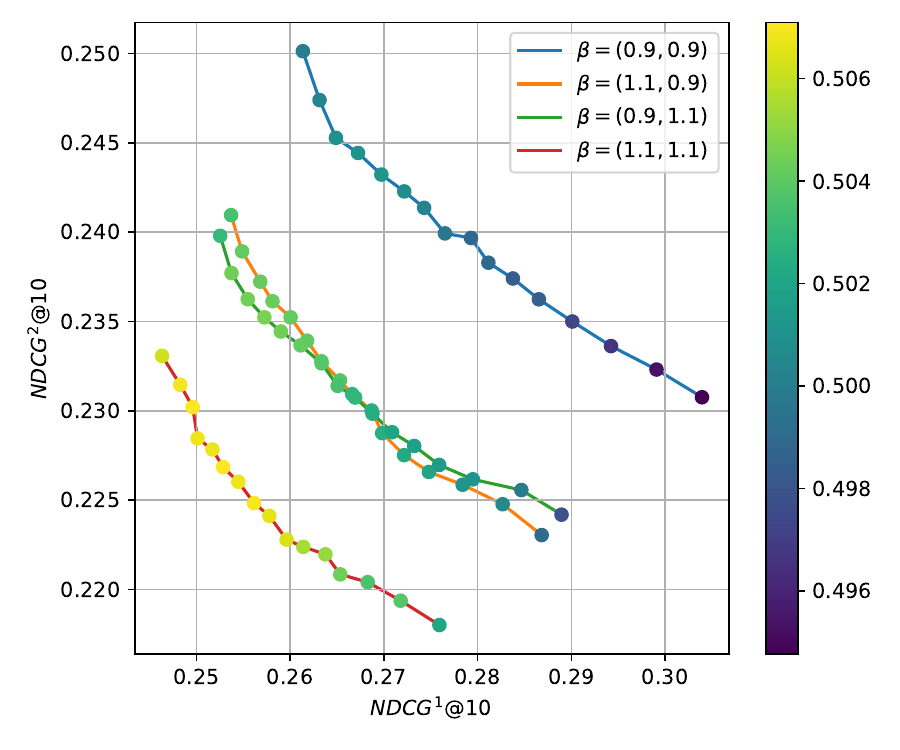}
        \caption{$\gD_\vbeta = \unif([0.83, 1.2]^2)$.}
        \label{fig:beta_83}
    \end{subfigure}
    \hfill
    \begin{subfigure}{.47\textwidth}
        \includegraphics[width=\linewidth]{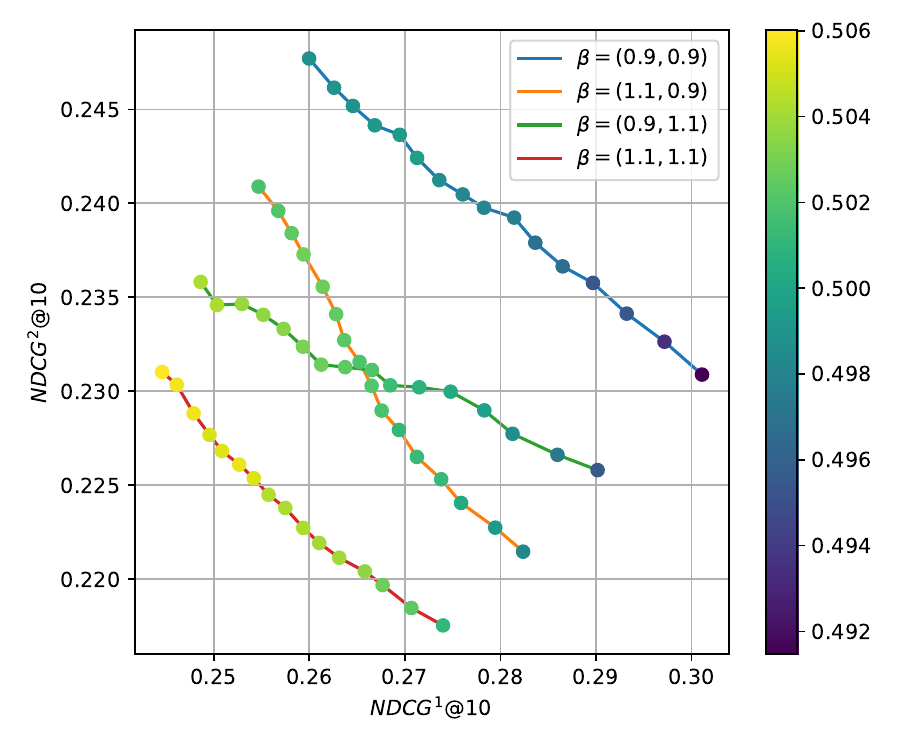}
        \caption{$\gD_\vbeta = \unif([0.71, 1.4]^2)$.}
        \label{fig:beta_71}
    \end{subfigure}
    
    \begin{subfigure}{.47\textwidth}
        \includegraphics[width=\linewidth]{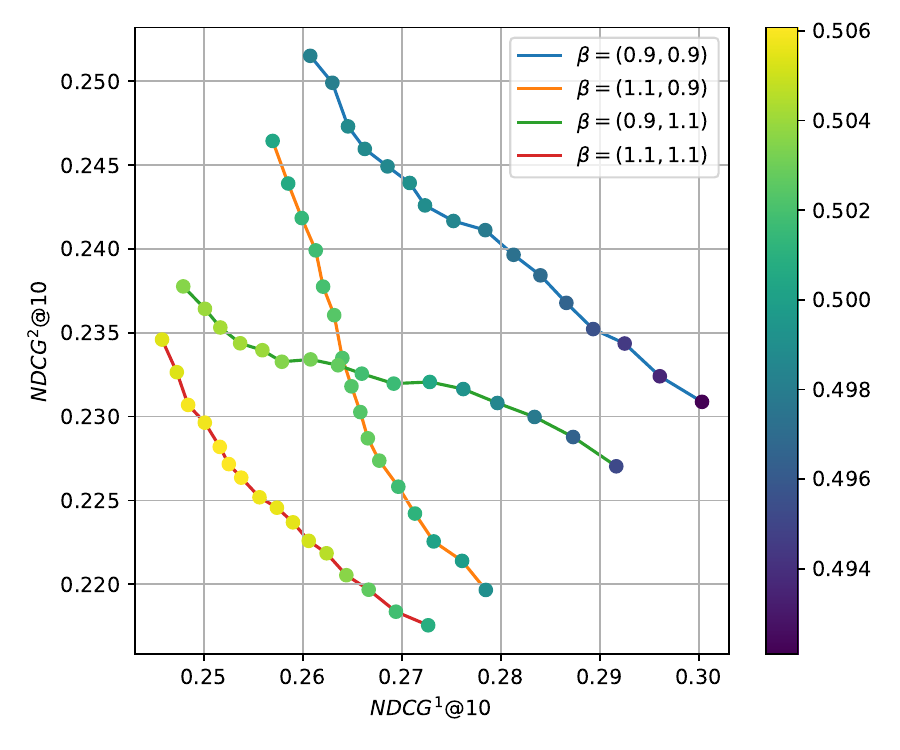}
        \caption{$\gD_\vbeta = \unif([0.63, 1.6]^2)$.}
        \label{fig:beta_63}
    \end{subfigure}
    \caption{Ablation study of the impact of the distribution $\gD_\vbeta$ on the Pareto fronts obtained by the Temperature-COS-DPO on the MSLR-WEB10K dataset (Objective IV vs Objective V). The colorbar denotes the NDCG@10 of the main objective.}
    \label{fig:temp_dist}
\end{figure}

The choice of the distribution $\gD_\vbeta$ also affects the performance of Temperature-COS-DPO. \cref{fig:temp_dist} shows the impact of the distribution $\gD_\vbeta$ on the Pareto fronts obtained by the Temperature-COS-DPO on the MSLR-WEB10K dataset. When the distribution $\gD_\vbeta$ only covers a small range, the Pareto fronts exhibit less expected behaviors and more noise, \emph{e.g.}, the Pareto fronts with $\gD_\vbeta = \unif([0.83, 1.2]^2)$ in \cref{fig:beta_83} are less concave and more scattered.
As the range of the distribution $\gD_\vbeta$ increases, the Pareto fronts become more concave and exhibit less desired behaviors. The results suggest that the distribution $\gD_\vbeta$ should cover a larger range than those interested to ensure sufficient training.

Given the results shown above and our studies on several hyperparameters, we conclude that despite requiring more expressive structures and more training resources, Temperature-COS-DPO is a feasible and promising direction for the COS-DPO framework to achieve more flexible control over the Pareto front and we expect to further investigate the validity of Temperature-COS networks and apply them to more complex multi-objective optimization tasks in the future.

\section{Missing Proofs}
\label{app:proofs}

In this section, we provide the proofs of the propositions and theorems mentioned in the main text.

\subsection{Proofs of Reparametrization-Related Arguments}
\label{app:reparam}

We prove the reparametrization of the DPO loss~\eqref{eq:dpo} and the LiPO loss~\eqref{eq:listnet_dpo} below.

\begin{proof}[Proof of~\eqref{eq:dpo}]
   Recall that in the second step of PPO, we consider the loss function~\eqref{eq:ppo_2} as follows:
   \begin{equation*}
        \begin{aligned}
            -\Ls(p_\theta; p_0, r_\phi, \beta) 
            &= \E_{(\vx, y) }\left[r_\phi(y | \vx) - \beta \log \dfrac{p_\theta(y|\vx)}{p_0(y|\vx)} \right]\\
            &= \int \left(r_\phi(y | \vx) - \beta \log \dfrac{p_\theta(y|\vx)}{p_0(y|\vx)} \right) p_\theta(y|\vx) \dif y,
        \end{aligned}
    \end{equation*}
    we calculate the functional derivative of the loss w.r.t. the density function $p_\theta(y|\vx)$:
    \begin{equation*}
        \begin{aligned}
            &\dfrac{\delta \Ls(p_\theta; p_0, r_\phi, \beta)}{\delta p_\theta(y|\vx)} 
            = \lim_{\epsilon \to 0} \dfrac{\Ls(p_\theta + \epsilon \delta p_\theta; p_0, r_\phi, \beta) - \Ls(p_\theta; p_0, r_\phi, \beta)}{\epsilon}\\
            =& \lim_{\epsilon \to 0}\dfrac{1}{\epsilon} \bigg[ \int \left(r_\phi(y | \vx) - \beta \log \dfrac{p_\theta(y|\vx)}{p_0(y|\vx)} - \beta \dfrac{\epsilon \delta p_\theta(y|\vx)}{p_\theta(y|\vx)} \right) (p_\theta(y|\vx) + \epsilon \delta p_\theta(y|\vx)) \dif y\\
            &\quad \quad - \int \left(r_\phi(y | \vx) - \beta \log \dfrac{p_\theta(y|\vx)}{p_0(y|\vx)} \right) p_\theta(y|\vx) \dif y \bigg]\\
            =& \int \left(r_\phi(y | \vx) - \beta \log \dfrac{p_\theta(y|\vx)}{p_0(y|\vx)} - \beta \right) \delta p_\theta(y|\vx) \dif y.
        \end{aligned}
    \end{equation*}

    Let the functional derivative vanish, we obtain 
    \begin{equation*}
        r_\phi(y | \vx) = \beta \log \dfrac{p_\theta(y|\vx)}{p_0(y|\vx)} + \beta,
    \end{equation*}
    \emph{i.e.},
    \begin{equation*}
        p_\theta(y|\vx)\ \propto\  p_0(y|\vx) \exp\left(\dfrac{r_\phi(y | \vx)}{\beta}\right).
    \end{equation*}

    Since the likelihood $\P(y_1 \succ y_2 | \vx)$~\eqref{eq:btl} in the BTL model only depends on the difference of the reward functions, $r_\phi(y | \vx)$ admits an arbitrary constant shift, and thus we assume $r_\phi(y | \vx)$ to be normalized in a way such that 
    \begin{equation*}
        \E\left[p_0(y|\vx) \exp\left(\dfrac{r_\phi(y | \vx)}{\beta}\right)\right] = 1,
    \end{equation*}
    which leads to the reparametrization $r_\theta(y|\vx) = \beta \log \frac{p_\theta(y|\vx)}{p_0(y|\vx)}$, plugging which into the PPO loss~\eqref{eq:ppo_2} yields the DPO loss~\eqref{eq:dpo}.
\end{proof}

\begin{proof}[Proof of~\eqref{eq:listnet_dpo}]
    
    As in the derivation of the DPO loss~\eqref{eq:dpo} under the BTL model, we first consider the PPO algorithm for the PL model:
    \begin{enumerate}[leftmargin=*, label={\it Step \arabic*.}, itemsep=0em]
        \item Find the optimal score function $s_\phi(\vy | \vx)$ that minimizes the ListNet loss~\eqref{eq:listnet}:
        \begin{equation}
            -\Ls_{\rm LN}(s_\theta; \gD_{\mathrm{LTR}}^j) = \E \left[\sum_{i=1}^{n} \overline z_{i}^{j} \log \dfrac{\exp(s_\phi(\vy_i| \vx))}{\sum_{i'=1}^{n} \exp(s_\phi(\vy_{i'}| \vx))}\right];
            \label{eq:ppo_1_listnet}
        \end{equation}
        \item Fine-tune the base model $s_0$ with the optimal score function $s_\phi$ by maximizing the expected score value while penalizing the KL divergence between the new model and the base model:
        \begin{equation}
            -\Ls(p_\theta; p_0, r_\phi, \beta) 
            = \E\left[s_\phi(\vy | \vx) \right] - \beta \KL(p_\theta|| p_0)
            = \E\left[s_\phi(\vy | \vx) - \beta \log \dfrac{p_\theta(\vy|\vx)}{p_0(\vy|\vx)} \right].
            \label{eq:ppo_2_listnet}
        \end{equation}
    \end{enumerate}

    For the optimization problem in the second step~\eqref{eq:ppo_2_listnet}, following the same procedure as in the proof of~\eqref{eq:dpo}, we solve the optimal $p_\theta$ by letting the functional derivative of the loss w.r.t. the density function $p_\theta(y|\vx)$ vanish and obtain
    \begin{equation}
        p_\theta(\vy|\vx)\ \propto\  p_0(\vy|\vx) \exp\left(\dfrac{s_\phi(\vy | \vx)}{\beta}\right).      
        \label{eq:prop_base}  
    \end{equation}
   
    By the assumption of the PL model and the ListNet loss, we have $p_\theta(\vy | \vx)$ modeled as the top-1 probability of the PL model and thus related to the score function $s_\theta(\vy | \vx)$ via
    \begin{equation*}
        p_\theta(\vy|\vx) = \dfrac{\exp(s_\theta(\vy | \vx))}{\sum_{i'=1}^{n} \exp(s_\theta(\vy_{i'}| \vx))}.
    \end{equation*}
    Let 
    $$p_0(\vy|\vx) = \frac{\exp(s_0(\vy | \vx))}{\sum_{i'=1}^{n} \exp(s_0(\vy_{i'}| \vx))},$$
    \eqref{eq:prop_base} can be rewritten as
    \begin{equation*}
        \exp(s_\theta(\vy | \vx))\ \propto\ \exp\left( s_0(\vy | \vx) + \beta s_\phi(\vy | \vx) \right),
    \end{equation*}
    \emph{i.e.},
    \begin{equation*}
        s_\theta(\vy | \vx) = s_0(\vy | \vx) + \beta s_\phi(\vy | \vx) + C,
    \end{equation*}
    where $C$ is a constant shift. By noticing that the softmax function in~\eqref{eq:ppo_1_listnet} is invariant to the constant shift of the score function $s_\phi(\vy | \vx)$, we may choose certain normalization such that
    \begin{equation*}
        s_\theta(\vy | \vx) = s_0(\vy | \vx) + \beta s_\phi(\vy | \vx)
    \end{equation*}
    holds, plugging which into the loss~\eqref{eq:ppo_1_listnet} yields the reparametrized ListNet loss~\eqref{eq:listnet_dpo}.
\end{proof}

\subsection{Proofs of Linear Transformation Property}
\label{app:linear}

In this section, we provide the proof of the linear transformation property of the Weight-COS-DPO loss. Instead of \cref{prop:linear} in the main text, we provide a more general proposition that considers the penalization terms in the Weight-COS-DPO loss introduced in \cref{app:penalization}. The takeaway of this generalization is that the linear transformation property still holds whenever the penalization term is a function of the normalized loss function $\Ls_{\rm LiPO}$.

\begin{proposition}[Linear Transformation Property with Penalization Terms]
    For any $\vbeta \in \R_+^m$ and $\vw\in\Delta^m$, we denote the model obtained by optimizing the Weight-COS-DPO loss~\eqref{eq:W-COS_loss} with temperature $\vbeta$ as $s_{\theta, \vbeta}(\vy, \vw | \vx)$, and suppose the penalization term $\gG_{\vw}(s_{\theta}; s_0, \vbeta)$ is a function of $\bm\Ls_{\rm LiPO}(s_\theta(\cdot, \vw | \vx); s_0, \vbeta, \gD_{\mathrm{MOFT}})$. 
    
    Then $s_{\theta, \vbeta}(\vy, \vw | \vx)$ should satisfy the linear transformation that for any $c > 0$, we have that
    \begin{equation}
        s_{\theta, c\vbeta}(\vy, \vw | \vx) = \left(1 - \frac{1}{c}\right) s_0(\vy | \vx) + \frac{1}{c} s_{\theta, \vbeta}(\vy, \vw | \vx)
        \label{eq:linear_transform_penalization}
    \end{equation}
    is also an optimal solution to the Weight-COS-DPO loss~\eqref{eq:W-COS_loss} with temperature $c\vbeta$.
    \label{prop:linear_penalization}
\end{proposition}

\begin{proof}[Proof of \cref{prop:linear_penalization}]

    For clarity, we first remove the penalization, \emph{i.e.} to consider the case where $\lambda = 0$.
    
    Then the Weight-COS-DPO loss~\eqref{eq:W-COS_loss} is of the following form:
    \begin{equation*}
        \begin{aligned}
            &\Ls_{\rm W\hyphen COS}(s_\theta(\cdot, \vw | \vx); s_0, \vbeta, \gD_{\mathrm{MOFT}}) \\
            =& \E_{\vw \sim \mathrm{Dir}(\valpha)}\left[\Ls_{\vw}(s_\theta(\cdot, \vw | \vx); s_0, \vbeta, \gD_{\mathrm{MOFT}})\right]\\
            =& \E_{\vw \sim \mathrm{Dir}(\valpha)}\left[\sum_{j=1}^m w_j \Ls_{\rm LiPO}(s_\theta(\cdot, \vw | \vx); s_0, \vbeta, \gD_{\mathrm{MOFT}}^j)\right]\\
            =&  \E_{\vw \sim \mathrm{Dir}(\valpha)}\left[\sum_{j=1}^m w_j \E\left[ \sum_{i=1}^n \overline z_i^j \log \left(\dfrac{ \exp\big(\beta_j(s_{\theta}(\vy_i, \vw | \vx) - s_0(\vy_i, \vw | \vx))\big)}{\sum_{i'=1}^n  \exp\big(\beta_j (s_{\theta}(\vy_{i'}, \vw | \vx) - s_0(\vy_{i'}, \vw | \vx))\big)}\right) \right]\right]\\
            =&  \E\left[\sum_{j=1}^m  \sum_{i=1}^n w_j \overline z_i^j \log \left(\dfrac{ \exp\big(\beta_j(s_{\theta}(\vy_i, \vw | \vx) - s_0(\vy_i, \vw | \vx))\big)}{\sum_{i'=1}^n  \exp\big(\beta_j (s_{\theta}(\vy_{i'}, \vw | \vx) - s_0(\vy_{i'}, \vw | \vx))\big)}\right) \right],
        \end{aligned}
    \end{equation*}
    where the expectation in the second to last equality is taken over the data distribution $\gD_{\mathrm{MOFT}}$, and the expectation in the last equality is taken over both the data distribution $\gD_{\mathrm{MOFT}}$ and the weight distribution $\mathrm{Dir}(\valpha)$ as a shorthand notation.

    By the definition of the model $s_{\theta, \vbeta}(\vy, \vw | \vx)$, we have that 
    \begin{equation*}
        \begin{aligned}
            s_{\theta, \vbeta}(\vy, \vw | \vx) 
            =& \argmax_{s_{\theta}(\vy, \vw | \vx)} \E\left[ \sum_{j=1}^m  \sum_{i=1}^n w_j \overline z_i^j \log \left(\dfrac{ \exp\big(\beta_j(s_{\theta}(\vy_i, \vw | \vx) - s_0(\vy_i, \vw | \vx))\big)}{\sum_{i'=1}^n  \exp\big(\beta_j (s_{\theta}(\vy_{i'}, \vw | \vx) - s_0(\vy_{i'}, \vw | \vx))\big)}\right) \right].
        \end{aligned}
    \end{equation*}

    We now consider the following reparametrized optimization problem:
    \begin{equation}
        \begin{aligned}
            \max_{s'_{\theta}(\vy, \vw | \vx)} 
            &\E\bigg[\sum_{j=1}^m  \sum_{i=1}^n w_j \overline z_i^j \log \left(\dfrac{ \exp\big(\beta_j(c s'_{\theta}(\vy_i, \vw | \vx) + (1 - c) s_{0}(\vy_i, \vw | \vx)- s_0(\vy_i, \vw | \vx))\big)}{\sum_{i'=1}^n  \exp\big(\beta_j (c s'_{\theta}(\vy, \vw | \vx) + (1 - c) s_{0}(\vy_{i'}, \vw | \vx) - s_0(\vy_{i'}, \vw | \vx))\big)}\right) \bigg]\\
            =& \E\left[ \sum_{j=1}^m  \sum_{i=1}^n w_j \overline z_i^j \log \left(\dfrac{ \exp\big(c\beta_j(s'_{\theta}(\vy_i, \vw | \vx) - s_0(\vy_i, \vw | \vx))\big)}{\sum_{i'=1}^n  \exp\big(c\beta_j (s'_{\theta}(\vy_{i'}, \vw | \vx) - s_0(\vy_{i'}, \vw | \vx))\big)}\right) \right],
        \end{aligned}
        \label{eq:linear_transform_2}
    \end{equation}
    obtained by reparametrizing $s_{\theta}(\vy, \vw | \vx)$ as 
    \begin{equation}
        s_{\theta}(\vy, \vw | \vx) = c s'_{\theta}(\vy, \vw | \vx) + (1 - c) s_{0}(\vy, \vw | \vx),
        \label{eq:linear_transform_reparam}
    \end{equation}
    and thus by solving
    \begin{equation*}
        s_{\theta, \vbeta}(\vy, \vw | \vx) = c s'_{\theta}(\vy, \vw | \vx) + (1 - c) s_{0}(\vy, \vw | \vx),
    \end{equation*}
    we have 
    \begin{equation}
        s'_{\theta}(\vy, \vw | \vx) = \dfrac{1}{c} s_{\theta, \vbeta}(\vy, \vw | \vx) - \dfrac{1 - c}{c} s_{0}(\vy, \vw | \vx)
        \label{eq:linear_transform_3}
    \end{equation}
    is an optimal solution to the reparametrized optimization problem.

    Notice that the function in the optimization problem~\eqref{eq:linear_transform_2} is exactly the Weight-COS-DPO loss~\eqref{eq:W-COS_loss} with the temperature $c\vbeta$, we have that the $s_{\theta, c \vbeta}(\vy, \vw | \vx)$ as defined in~\eqref{eq:linear_transform} coincides with the optimal solution~\eqref{eq:linear_transform_3}. Thus we have proved the linear transformation property for the Weight-COS-DPO loss with $\lambda = 0$.
    
    For the case with penalization, we assume the penalization term $\gG_{\vw}(s_{\theta}; s_0, \vbeta)$ is a function of the vector of LiPO losses $\bm\Ls_{\rm LiPO}(s_\theta(\cdot, \vw | \vx); s_0, \vbeta, \gD_{\mathrm{MOFT}})$, which is satisfied for the cosine similarity penalization loss~\eqref{eq:penalization} as proposed by~\citet{ruchte2021scalable}. And in turn, the vector of LiPO losses $\bm\Ls_{\rm LiPO}(s_\theta(\cdot, \vw | \vx); s_0, \vbeta, \gD_{\mathrm{MOFT}})$ depends on $s_{\theta}(\cdot, \vw | \vx)$ only in the form of $s_{\theta}(\cdot, \vw | \vx) - s_0(\cdot, \vw | \vx)$, and therefore, we could write the Weight-COS-DPO loss in an abstract form as 
    \begin{equation*}
        s_{\theta, \vbeta}(\vy, \vw | \vx) 
        = \argmax_{s_{\theta}(\vy, \vw | \vx)} \E\left[ \Phi\left(s_{\theta}(\cdot, \vw | \vx) - s_0(\cdot, \vw | \vx)\right) \right],
    \end{equation*}
    \emph{e.g.}, for the case where $\lambda = 0$, $\Phi$ is of the following form:
    \begin{equation*}
        \begin{aligned}
            &\Phi(s_{\theta}(\cdot, \vw | \vx) - s_0(\cdot, \vw | \vx)) \\
            =& 
            \sum_{j=1}^m  \sum_{i=1}^n w_j \overline z_i^j \log \left(\dfrac{ \exp\big(\beta_j(s_{\theta}(\vy_i, \vw | \vx) - s_0(\vy_i, \vw | \vx))\big)}{\sum_{i'=1}^n  \exp\big(\beta_j (s_{\theta}(\vy_{i'}, \vw | \vx) - s_0(\vy_{i'}, \vw | \vx))\big)}\right).
        \end{aligned}
    \end{equation*}

    Apply the same reparametrization as in~\eqref{eq:linear_transform_reparam}, we have that the reparametrized optimization problem is of the form
    \begin{equation*}
        \begin{aligned}
            \max_{s'_{\theta}(\vy, \vw | \vx)} 
            &\E\left[ \Phi\left(c s'_{\theta}(\cdot, \vw | \vx) + (1 - c) s_{0}(\cdot, \vw | \vx) - s_0(\cdot, \vw | \vx)\right) \right]\\
            =& \E\left[ \Phi\left(c s'_{\theta}(\cdot, \vw | \vx) - c s_{0}(\cdot, \vw | \vx)\right) \right],
        \end{aligned}
    \end{equation*}
    with an optimal solution in the form of~\eqref{eq:linear_transform_3}. Therefore, the linear transformation property also holds for the Weight-COS-DPO loss with the penalization term.
\end{proof}

\end{document}